\documentclass[sigconf]{acmart}
\AtBeginDocument{%
  }

\newcommand{\hao}[1]{#1}
\newcommand{\beftext}[1]{}
\newcommand{\eat}[1]{}

\newcommand{\song}[1]{{\color{brown}{#1}}}
\usepackage{multirow}
\usepackage[table,xcdraw]{xcolor}

\copyrightyear{2026}
\acmYear{2026}
\setcopyright{cc}
\setcctype{by}
\acmConference[KDD '26]{Proceedings of the 32nd ACM SIGKDD Conference on Knowledge Discovery and Data Mining V.2}{August 09--13, 2026}{Jeju Island, Republic of Korea}
\acmBooktitle{Proceedings of the 32nd ACM SIGKDD Conference on Knowledge Discovery and Data Mining V.2 (KDD '26), August 09--13, 2026, Jeju Island, Republic of Korea}
\acmDOI{10.1145/3770855.3818002}
\acmISBN{979-8-4007-2259-2/2026/08}

\usepackage[most]{tcolorbox}
\usepackage{xcolor}
\usepackage{alltt}
\usepackage{listings}
\usepackage{subcaption}

\begin{document}


\title{EvoDS: Self-Evolving Autonomous Data Science Agent with Skill Learning and Context Management}


\author{Zherui Yang}
\affiliation{%
  \institution{The Hong Kong University of Science and Technology (Guangzhou)}
  \city{Guangzhou}
  \country{China}
}
\email{zyang582@connect.hkust-gz.edu.cn}

\author{Fan Liu}
\affiliation{%
  \institution{The Hong Kong University of Science and Technology (Guangzhou)}
  \city{Guangzhou}
  \country{China}}
\email{fliu236@connect.hkust-gz.edu.cn}

\author{Yansong Ning}
\affiliation{%
  \institution{The Hong Kong University of Science and Technology (Guangzhou)}
  \city{Guangzhou}
  \country{China}}
\email{yning092@connect.hkust-gz.edu.cn}

\author{Hao Liu}
\authornote{Corresponding author.}
\affiliation{%
  \institution{The Hong Kong University of Science and Technology (Guangzhou)}
  \city{Guangzhou}
  \country{China}}
\email{liuh@ust.hk}

\renewcommand{\shortauthors}{Zherui Yang, Fan Liu, Yansong Ning, and Hao Liu}


\tcbset{
  aibox/.style={
    width=400.18663pt,
    top=10pt,
    colback=black!05,
    colframe=black!20,
    colbacktitle=black!25!gray,
    enhanced,
    center,
    attach boxed title to top center={yshift=-0.1in},
    boxed title style={boxrule=0pt,colframe=white,},
  }
}

\tcbset{
  aiboxnotitle/.style={
    width=400.18663pt,
    top=10pt,
    colback=black!05,
    colframe=black!20,
    enhanced,
    center,
  }
}

\tcbset{
  aiboxbreakable/.style={
    width=400.18663pt,
    top=10pt,
    colback=black!05,
    colframe=black!20,
    colbacktitle=black!50,
    enhanced,
    center,
    breakable,
    attach boxed title to top left={yshift=-0.1in,xshift=0.15in},
    boxed title style={boxrule=0pt,colframe=white,},
  }
}
\newtcolorbox{AIBox}[2][]{aibox,title=#2,#1}
\newtcolorbox{AIBoxNoTitle}[1][]{aiboxnotitle}
\newtcolorbox{AIBoxBreak}[2][]{aiboxbreakable,title=#2,#1}

\lstset{
  basicstyle=\scriptsize\ttfamily,
  breaklines=true,
  columns=fullflexible
}
\lstset{linewidth=380pt}

\newtheorem{theoremBrief}{Theorem}
\newtheorem{assumption}{Assumption}

\eat{
\begin{abstract}
Large language model-based agents show promise for automating data science workflows, yet they remain constrained by limited and inflexible action spaces as well as challenges in long-horizon reasoning. To address these limitations, we propose \textbf{EvoDS}, a self-evolving data science agent built upon a hierarchical multi-agent framework. To enable flexible and scalable action spaces, EvoDS introduces an Adaptive Tool Evolution Mechanism that synthesizes and accumulates task-specific tools on demand. To support stable long-horizon execution, we further propose a Context Compression Strategy that autonomously retains task-relevant information while discarding redundant context. In addition, we develop a multi-agent reinforcement learning algorithm that jointly optimizes task performance, action space evolution, and context management within a unified training process. Both theoretical analysis and extensive empirical evaluations demonstrate the effectiveness of the proposed approach, with EvoDS consistently outperforming strong open-source and proprietary baselines across four data science benchmarks. The code is available at \url{https://anonymous.4open.science/r/EvoDS}.
\end{abstract}}

\begin{abstract}

\hao{
Recent progress in Large Language Model~(LLM) agents has enabled promising advances in automated data science. However, existing approaches remain fundamentally limited by their static action sets and lack of principled long-horizon context management, hindering their ability to accumulate reusable experience across tasks and operate reliably in multi-stage, iterative data science pipelines.
To address these challenges, we introduce EvoDS, a self-evolving autonomous data science agent that learns to expand its skills and adaptively managing long-term context through agentic reinforcement learning. 
Specifically, EvoDS introduces two key strategies: (1) Autonomous Skill Acquisition (ASA) mechanism, which enables agents to synthesize, validate, and reuse executable  skills; and (2) Adaptive Context Compression (ACC) strategy, which treats context management as a learned control problem rather than passive truncation.
These strategies are orchestrated within a two-stage multi-agent training scheme, enabling EvoDS to autonomously improve over time.
Theoretically, we prove that EvoDS’s hierarchical design reduces tool-selection error, and its optimization objective aligns with an information bottleneck principle, ensuring efficient context use. Empirically, EvoDS outperforms state-of-the-art open-source data science agents by an average of 28.9\% across four diverse benchmarks while eliminating out-of-token failures.
Our code and data are available at \url{https://github.com/usail-hkust/EvoDS}.
}

\end{abstract}


\eat{

}


\maketitle

\section{Introduction}

\hao{
Automating data science workflows, from data preprocessing and feature engineering to model selection, evaluation, and visualization, has long been a central goal of machine learning research, driven by the growing demand for scalable, accessible, and efficient analytics across scientific and industrial domains~\cite{DBLP:journals/cacm/BieRHHSW22, MUMUNI2025113, DBLP:journals/jair/ZollerH21}. Traditional AutoML systems have made important progress in automating isolated stages of this pipeline, yet they typically rely on rigid search spaces and predefined operators, limiting flexibility in open-ended analytical tasks~\cite{DBLP:journals/jair/ZollerH21, he2021automl}. Recent advances in large language models (LLMs) have renewed interest in autonomous data science agents capable of reasoning over natural language instructions, writing executable code, invoking tools, and coordinating multi-step workflows~\cite{DBLP:journals/fcsc/WangMFZYZCTCLZWW24, DBLP:conf/ijcai/GuoCWCPCW024, DBLP:conf/kdd/0002CWW00Z0DLP025}. By integrating language reasoning with perception, memory, and action, LLM-based agents offer a promising pathway toward end-to-end automation of realistic data science pipelines with minimal human intervention~\cite{sun2025survey, DBLP:journals/corr/abs-2508-02744, DBLP:journals/corr/abs-2509-23988, DBLP:journals/corr/abs-2510-23587, DBLP:journals/corr/abs-2510-04023}.

Despite this promise, existing autonomous data science agents remain largely constrained by predefined workflows and fixed action spaces~\cite{DBLP:conf/icml/GuoD0C0024, DBLP:conf/icml/TriratJH25, li2025autokaggle, DBLP:journals/corr/abs-2510-16872}. Recent systems such as DS-Agent~\cite{DBLP:conf/icml/GuoD0C0024}, AutoKaggle~\cite{li2025autokaggle}, and DeepAnalyze~\cite{DBLP:journals/corr/abs-2510-16872} demonstrate encouraging results by orchestrating LLMs with code execution and tool calling for tasks including feature engineering, model training, and evaluation. However, these approaches primarily rely on scripted pipelines or hand-designed control logic, treating tools as static primitives rather than learnable capabilities. Consequently, agents repeatedly rediscover similar solutions through trial-and-error, instead of abstracting successful behaviors into reusable skills. Moreover, prior experience is typically discarded after task completion, preventing systematic improvement over time. This stands in contrast to the problem-solving process of human data scientists, which is inherently exploratory, iterative, and experience-driven. These observations raise a fundamental research question: \emph{can we design autonomous data science agents that not only execute tasks, but also support open-ended exploratory analysis and systematically accumulate experience across iterations and tasks?}

However, achieving this goal introduces two tightly coupled challenges. 
(1) \textbf{Automatic Skill Acquisition}: How can data science agents acquire reusable skills from experience, abstracting successful problem-solving procedures into persistent actions or operators that enable self-improvement across tasks? Without such skill acquisition, agents remain confined to one-off reasoning and cannot systematically benefit from prior exploration. 
(2) \textbf{Explosive Context Management}: How can agents effectively manage rapidly growing context arising from iterative experimentation, intermediate artifacts, and evolving skills? Data science workflows naturally require long-horizon iterative reasoning, while continual skill evolution further enlarges the action space and execution context, aggravating long-context challenges. Such context explosion not only induces the well-known lost-in-the-middle phenomenon~\cite{DBLP:journals/tacl/LiuLHPBPL24}, degrading long-horizon reasoning performance, but also hampers effective experience accumulation by obscuring which actions and decisions were truly responsible for success~\cite{DBLP:journals/corr/abs-2508-02744, DBLP:journals/corr/abs-2510-23587}. 
Consequently, autonomous data science demands principled mechanisms for jointly learning skills and regulating memory, allowing agents to preserve task-critical information while suppressing irrelevant details under bounded context constraints. 
}

\hao{
To address the above challenges, we propose \textbf{EvoDS}, a self-evolving autonomous data science agent that integrates skill acquisition and context regulation within a hierarchical multi-agent architecture. EvoDS employs a Manager Agent to coordinate specialized agents for subtasks such as data handling, modeling, visualization, and debugging, where each agent maintains scope-specific data science skills. This design decomposes complex workflows into atomic executable subtasks for fine-grained skill evolution, while localizing long execution contexts for more effective context management. To maintain consistency, agents share a global memory for overall task objectives and maintain local memories for subtask-specific context. To enable experience-driven self-improvement, EvoDS introduces an \emph{Autonomous Skill Acquisition} (ASA) mechanism that abstracts successful problem-solving behaviors into reusable executable skills, allowing the action space to progressively expand over time. EvoDS further incorporates an \emph{Adaptive Context Compression} (ACC) strategy to selectively retain task-critical information while suppressing irrelevant details under bounded context windows. The above components are jointly optimized through a two-stage agentic reinforcement learning (RL) framework. We first collect trajectories from a teacher model for supervised fine-tuning (SFT), and then perform online RL to jointly optimize task performance, skill acquisition, and context management. This design enables EvoDS to simultaneously learn autonomous task execution, progressive skill acquisition, and long-horizon context management within a unified multi-agent framework.


Our main contributions are summarized as follows:
\begin{itemize}
\item We propose EvoDS, a self-evolving autonomous data science agent that integrates Autonomous Skill Acquisition mechanism and Adaptive Context Compression Strategy within a unified hierarchical multi-agent framework, enabling agents to progressively acquire reusable operational skills while actively controlling long-term context.
\item We design a two-stage multi-agent training scheme that jointly optimizes task execution, skill acquisition, and context regulation, allowing EvoDS to improve with experience while remaining robust under context constraints.
\item 
We theoretically show that the hierarchical architecture reduces tool-selection errors and that the reinforcement learning objective aligns with an information bottleneck, encouraging retention of task-critical information while filtering irrelevant signals.
\item 
Extensive experiments on four benchmarks demonstrate that EvoDS significantly outperforms state-of-the-art open-source agents, achieves robust long-horizon performance without out-of-token failures, and exhibits consistent improvement from accumulated experience.
\end{itemize}
}

\section{Related Works}
\subsection{Data Science Agents}

Recent advances in LLMs have driven the development of autonomous data science agents, which aim to automate a wide range of data-centric tasks, including exploratory data analysis, feature construction, predictive modeling, and result interpretation~\cite{zhang2024datacopilot, DBLP:conf/www/ChenYZ000HMZ25, DBLP:journals/corr/abs-2601-10402, DBLP:journals/corr/abs-2506-16499, DBLP:journals/corr/abs-2510-08511, nam2025mlestar, fang2025mlzero}. Early approaches mainly rely on workflow-based paradigms with predefined pipelines~\cite{DBLP:conf/icml/GuoD0C0024, DBLP:conf/icml/TriratJH25, li2025autokaggle, liu2025mmagent}. For example, DS-Agent~\cite{DBLP:conf/icml/GuoD0C0024} follows a pipeline that retrieves relevant cases and iteratively generates and debugs code. More recent work shifts toward fully autonomous agents that can dynamically plan and execute actions based on the current state~\cite{DBLP:journals/corr/abs-2510-16872, DBLP:journals/corr/abs-2505-23723, DBLP:journals/corr/abs-2509-25084}. For instance, DeepAnalyze~\cite{DBLP:journals/corr/abs-2510-16872} defines a set of actions and enables the agent to iteratively reason and act. Despite these advances, existing approaches remain constrained by predefined workflows or fixed action spaces, treating actions as static primitives rather than learnable skills, and thus failing to abstract reusable skills from experience. Additionally, they lack effective long-context management for inherently long-horizon data science tasks.

\subsection{Self-Evolving Strategies in LLM Agents}
Due to the large parameter scale of LLMs, many studies explore self-evolving strategies for LLM agents at inference time~\cite{DBLP:journals/corr/abs-2507-21046, DBLP:journals/corr/abs-2508-07407}, improving prompts, memories, tools, skills, or agent frameworks without updating model parameters. Prompt-based evolution methods iteratively refine prompts or reasoning templates according to task feedback, thereby improving task-solving performance and reasoning quality~\cite{DBLP:journals/corr/abs-2406-07496, DBLP:conf/iclr/Guo0GLS0L0Y24, DBLP:conf/icml/FernandoBMOR24}. Memory-based approaches accumulate and retrieve experiences during task execution to support continual improvement and long-term adaptation~\cite{xu2025amem, DBLP:conf/aaai/ZhongGGYW24, DBLP:journals/corr/abs-2504-19413, DBLP:conf/icml/WangMFN25}. Tool- or skill-based evolution methods further enable agents to synthesize, modify, and reuse executable tools during problem solving~\cite{DBLP:conf/iclr/YuanC000J24, DBLP:journals/corr/abs-2603-01145}. In addition, topology-based evolution in multi-agent systems dynamically adjusts agent coordination structures and collaboration strategies for different task types~\cite{DBLP:conf/icml/ZhugeWKFKS24, DBLP:conf/iclr/ZhangXYTCCZCHWZ25}. Among these directions, tool- and skill-based evolution is most related to our work, while EvoDS further integrates skill acquisition and context regulation within a unified architecture for challenging data science tasks.

\eat{
Recent studies have explored self-evolving strategies for autonomous agents, enabling capability adaptation during inference based on interaction experience and environmental feedback without updating LLM parameters~\cite{DBLP:journals/corr/abs-2507-21046, DBLP:journals/corr/abs-2508-07407}. Existing approaches span several paradigms, including memory-based evolution for experience accumulation and reuse~\cite{xu2025amem, DBLP:conf/aaai/ZhongGGYW24, DBLP:journals/corr/abs-2504-19413, DBLP:conf/icml/WangMFN25}, prompt-based evolution for refining prompts or reasoning templates~\cite{DBLP:journals/corr/abs-2406-07496, DBLP:conf/iclr/Guo0GLS0L0Y24, DBLP:conf/icml/FernandoBMOR24}, tool-based evolution for synthesizing, modifying, or reusing executable tools~\cite{DBLP:conf/iclr/YuanC000J24, lu2026dont}, and topology-based evolution in multi-agent systems to improve coordination~\cite{DBLP:conf/icml/ZhugeWKFKS24, DBLP:conf/iclr/ZhangXYTCCZCHWZ25, DBLP:conf/icml/ZhangNF00025, DBLP:conf/iclr/ShangLZMLXL25}. Among them, tool-based evolution enables execution-based evaluation of acquired capabilities, whereas other paradigms are often harder to assess systematically. Consequently, we adopt tools as explicit carriers of agent capabilities for self-evolving data science agents.}

\subsection{Context Compression for LLM Agents.}
Context compression is essential for LLM-based agents to support long-horizon reasoning and multi-turn interactions~\cite{DBLP:journals/corr/abs-2507-13334}. Existing methods mainly follow two strategies. One line of work compresses context through summarization once the token length exceeds predefined limits, reducing memory usage while preserving coarse-grained information~\cite{DBLP:conf/icml/LeeCFCF24, DBLP:conf/acl/FeiNZH0D024, DBLP:journals/corr/abs-2510-00615}. Another line of work decomposes complex tasks into subtasks and retains only task-critical information by folding intermediate execution processes, thereby alleviating long-context interference during multi-step reasoning~\cite{DBLP:journals/corr/abs-2510-11967, DBLP:journals/corr/abs-2512-22733}. Different from prior work, EvoDS introduces an adaptive context compression strategy that dynamically determines when to compress context and what information to retain.

\eat{
Context management is essential for LLM-based agents to support long-horizon reasoning and multi-turn interactions~\cite{DBLP:journals/corr/abs-2507-13334}. Existing approaches can be broadly categorized into two paradigms. Memory-based methods maintain a memory module to store past interactions or salient information, enabling future reasoning under context length constraints~\cite{DBLP:conf/acl/ChenLPSSZGY25, zhang2025gmemory, DBLP:conf/icml/LeeCFCF24}. In contrast, context compression methods enable long-horizon reasoning by either explicitly abstracting historical interactions into compact summaries~\cite{DBLP:conf/icml/LeeCFCF24, DBLP:conf/acl/FeiNZH0D024, DBLP:journals/corr/abs-2510-00615} or by reducing computational overhead~\cite{DBLP:conf/iclr/00010WWCW24, he2024camelot}. In this work, we adopt tools as carriers of cross-task capabilities and employ context compression for intra-task context management, avoiding the additional complexity introduced by explicit memory retrieval.}

\begin{figure*}[!t]
  \centerline{\includegraphics[width=\linewidth]{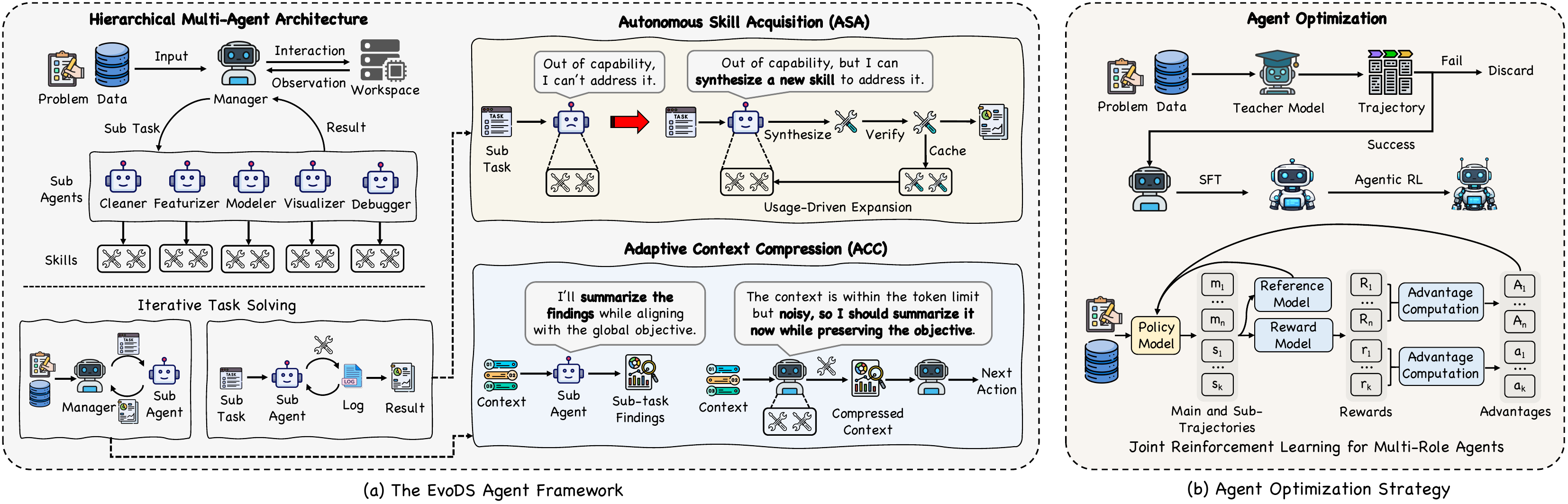}}
  \caption{Overview of EvoDS. (a) EvoDS adopts a hierarchical multi-agent architecture with autonomous skill acquisition and adaptive context compression for data science tasks. (b) The agent is trained via SFT and agentic RL to jointly optimize task execution, skill acquisition, and context management.}
  \label{fig:framework}
\end{figure*}

\subsection{Agent Optimization}
Agent optimization mainly follows two paradigms: SFT and RL. SFT improves agent capabilities by imitating expert trajectories or high-quality demonstrations, providing stable initialization for downstream decision making and tool use. RL further enables agents to learn from environment feedback through trial-and-error interactions, improving exploration, robustness, and adaptability in dynamic environments~\cite{DBLP:journals/corr/abs-2509-02547}. Most prior work focuses on single-agent optimization~\cite{DBLP:journals/corr/abs-2505-23723, DBLP:journals/corr/abs-2509-25084, DBLP:journals/corr/abs-2510-16872}, using algorithms such as PPO~\cite{DBLP:journals/corr/SchulmanWDRK17}, RLOO~\cite{DBLP:conf/acl/AhmadianCGFKPUH24}, GRPO~\cite{DBLP:journals/corr/abs-2402-03300}, and REINFORCE++~\cite{DBLP:journals/corr/abs-2501-03262} with sparse reward signals (e.g., answer correctness). More recently, multi-agent RL has been explored to enable agent specialization and coordination~\cite{DBLP:journals/corr/abs-2511-13288, DBLP:journals/corr/abs-2510-04678, DBLP:conf/nips/Bo0DFW00W24, motwani2025malt, DBLP:conf/acl/Park0GOZK25}. However, existing methods typically assume fixed environments and static action spaces. In contrast, EvoDS jointly optimizes task performance, skill acquisition, and long context management within a unified multi-agent framework.

\section{Preliminary}

\begin{definition}\textbf{Agent Skill.}
An agent skill is defined as $a = \langle n, d, c \rangle$, where $n$ denotes the skill name, $d$ is a textual description specifying the functionality and usage instructions of the skill, and $c$ denotes the corresponding executable code implementation.
\end{definition}

\begin{definition}\textbf{Agent Context.}
The context of an agent at step $t$ is defined as $C_t$, which contains the task description and the historical interaction information accumulated between the agent and the environment during the previous $t-1$ steps.
\end{definition}

\begin{definition}\textbf{Data Science Agent.}
A data science agent $\pi_\theta$ is an LLM-based agent parameterized by $\theta$, equipped with an action space $\mathcal{A}$ (e.g., code execution, tool invocation, or textual response). At step $t$, the agent selects an action $a_t \in \mathcal{A}$ based on the current context $C_t$, receives feedback from the environment, and updates its context to $C_{t+1}$.
\end{definition}

In this work, skills are treated as a type of action that can be invoked by the agent in the form of executable tools.

\begin{definition}\textbf{Task for Data Science Agent.}
Given a data science task $\mathcal{T} = \langle q, \mathcal{D} \rangle$, where $q$ is a natural language task description and $\mathcal{D} = \{d_1, \dots, d_N\}$ denotes the associated dataset, which may consist of multiple files in arbitrary formats, the goal of the agent is to generate a solution $\hat{y}$ through multi-step interaction with the environment. The quality of the solution is evaluated by a task-specific function $R_{\text{outcome}} = \mathcal{F}(\hat{y})$. For tasks with ground truth, $\mathcal{F}$ directly compares $\hat{y}$ with the reference answer; for open-ended tasks (e.g., machine learning or visualization), it may be defined via relative ranking or LLM-as-a-judge, depending on the benchmark.
\end{definition}

\section{Methodology}
\eat{In this section, we present EvoDS, an autonomous data science agent designed to tackle complex, long-horizon data science tasks.} 
The overall framework of EvoDS is illustrated in Figure~\ref{fig:framework}. We first introduce its hierarchical multi-agent architecture, followed by an \emph{Autonomous Skill Acquisition} mechanism for on-demand skill expansion and a \emph{Adaptive Context Compression} strategy for effective context management. Finally, we present a reinforcement learning strategy for multi-role agents that jointly optimizes task performance, skill acquisition, and context management.

\subsection{Hierarchical Multi-Agent Architecture}

Data science workflows involve diverse stages such as data cleaning, feature engineering, and model development. Directly handling all these operations within a single agent often leads to excessive context accumulation and large action spaces, as the agent must simultaneously maintain long execution histories and diverse domain-specific skills. Such challenges significantly increase reasoning complexity and aggravate long-context interference.

To address these issues, EvoDS adopts a \emph{Hierarchical Multi-Agent Architecture} that decomposes data science workflows across specialized agents with localized skill spaces. At the top level, a \emph{Manager Agent} serves as the global controller responsible for high-level reasoning, task decomposition, and inter-agent coordination. At lower levels, specialized agents execute concrete subtasks within their respective domains. To maintain alignment with the overall task objective, each agent maintains two types of memory: a \emph{global memory} that stores overall task goals, and a \emph{local memory} that records subtask-specific objectives and intermediate execution context. Such a design enables agents to focus on localized reasoning while remaining aligned with the global objective.

Specifically, the Manager interacts directly with the execution environment using a minimal set of code execution tools $\mathcal{A}^{\text{man}}$ for environment exploration, data inspection, and program execution. Based on the global state, the Manager either performs high-level reasoning directly or delegates subtasks to specialized sub-agents. Furthermore, based on the characteristics of data science workflows, EvoDS employs multiple sub-agents, including a \emph{Cleaner}, \emph{Featurizer}, \emph{Modeler}, \emph{Visualizer}, and \emph{Debugger}, responsible for data cleaning, feature engineering, model development, visualization, and debugging, respectively. These sub-agents also function as callable tools for the Manager agent. Formally, invoking a sub-agent is treated as a high-level action:
\begin{equation}
a_i^{\text{sub}} \in \mathcal{A}^{\text{man}}, \quad
a_i^{\text{sub}} = \text{Invoke}(\pi^i_\theta, q_j),
\end{equation}
where $\pi^i_\theta$ denotes the $i$-th sub-agent and $q_j$ denotes the $j$-th subtask assigned to $\pi^i_\theta$.

To initialize agent capabilities, following AutoKaggle~\cite{li2025autokaggle} and ML-Tool-Bench~\cite{DBLP:journals/corr/abs-2512-00672}, we construct a set of basic data science skills that support common operations such as preprocessing, training, and visualization, and organize these skills according to sub-agent expertise. Each sub-agent $\pi^i_\theta$ maintains a localized action space $\mathcal{A}_i$:
\begin{equation}
\mathcal{A}_i \cap \mathcal{A}_j = \emptyset (i \neq j),
\end{equation}
which reduces reasoning complexity and improves decision effectiveness by restricting agents to domain-specific operations.

Under this hierarchical framework, the Manager iteratively selects actions according to the global workflow state, either interacting with the environment directly or delegating subtasks to specialized sub-agents. Sub-agents execute scope-specific operations within their localized skill spaces, return execution results to the Manager, and the process repeats until the task is completed.

\subsection{Autonomous Skill Acquisition}

To enable agents to acquire reusable skills from experience and support self-evolving behavior across tasks, we introduce an \emph{Autonomous Skill Acquisition} mechanism. During task execution, the Manager assigns subtasks to corresponding sub-agents according to their expertise scopes. However, assigned subtasks may exceed the existing skill space of a sub-agent. In such cases, the sub-agent can synthesize new executable skills to solve the given problem. Specifically, the proposed mechanism consists of four stages: \emph{Synthesis}, \emph{Verification}, \emph{Caching}, and \emph{Expansion}.

\textbf{Synthesis.}
In the synthesis stage, the agent uses the underlying LLM to generate a new executable skill $a_{\text{new}} = \langle n, d, c \rangle$ through prompting, where $n$ denotes the skill name, $d$ is a textual description specifying usage instructions, and $c$ denotes the corresponding executable code. Such a representation enables structured and parameterized tool invocation. The synthesized skill is designed to be task-agnostic and reusable across different tasks.

\textbf{Verification.}
In the verification stage, the agent invokes the synthesized skill according to its usage description $d$ and executes the corresponding code $c$ within the environment to solve the subtask. The skill is evaluated through execution feedback, including executability and output validity. Only skills that execute successfully and produce valid outputs are regarded as effective, while skills that fail or generate invalid outputs are discarded.

\textbf{Caching.}
Naively adding all synthesized skills into a sub-agent's action space may introduce a large number of infrequently used or low-quality skills, thereby degrading skill-selection accuracy. To address this issue, all validated skills are first stored in a synthesized skill repository during the caching stage. We denote the synthesized skill repository for the $i$-th sub-agent as:
\begin{equation}
\Delta \mathcal{A}_i = \{ a_{\text{new}} \mid a_{\text{new}} \notin \mathcal{A}_i \}.
\end{equation}

\textbf{Expansion.}
In the expansion stage, EvoDS adopts a usage-frequency-aware expansion strategy. Specifically, skills are identified by their names $n$, and repeated synthesis of the same skill indicates recurring capability requirements. We maintain a generation count $c(a_{\text{new}})$ for each synthesized skill $a_{\text{new}}$, and permanently add the skill into the sub-agent's action space only when its count exceeds a predefined threshold $\tau$:
\begin{equation}
\mathcal{A}_i \leftarrow \mathcal{A}_i \cup \{ a_{\text{new}} \in \Delta \mathcal{A}_i \mid c(a_{\text{new}}) \ge \tau \}.
\end{equation}
In practice, we set $\tau = 3$.

Through this mechanism, EvoDS incrementally acquires new skills during problem solving and externalizes them as reusable executable skills. This design enables the agent to autonomously identify and address capability gaps over time, supporting continual adaptation to diverse and evolving data science tasks.

\subsection{Adaptive Context Compression}

Data science workflows typically involve multiple stages, which naturally produce long execution contexts containing data previews, code snippets, tool invocations, and execution logs. Moreover, the proposed \emph{Autonomous Skill Acquisition} mechanism continuously expands the available skill space during execution, further aggravating context growth. Such long contexts pose significant challenges to LLM reasoning. To address these challenges, we introduce a two-level \emph{Adaptive Context Compression} strategy.

For each sub-agent, it maintains two types of memory: a shared \emph{global memory} containing the overall task objective, and a \emph{local memory} containing subtask-specific goals and execution context. During execution, each sub-agent iteratively solves assigned subtasks and produces raw execution results, which may contain extensive intermediate information. Instead of directly returning raw outputs to the Manager Agent, sub-agents leverage the underlying LLM to distill execution results into concise summaries conditioned on the global task objective, thereby preserving task-relevant information while filtering unnecessary details. Specifically, successful executions are summarized by their key outcomes, while failed executions are summarized by their failure causes and error patterns. Formally, the compression process is defined as $\tilde{o}_t = \phi(o_t \mid G)$, where $o_t$ denotes the raw execution result, $G$ denotes the global task objective, and $\phi(\cdot)$ is a compression function that preserves critical information while discarding redundant details. The compressed summaries $\tilde{o}_t$ are then returned to the Manager Agent, reducing context length while maintaining semantic fidelity.

For the Manager Agent, in contrast to passive compression triggered only when the context length exceeds a predefined threshold~\cite{DBLP:journals/corr/abs-2510-00615, DBLP:journals/corr/abs-2509-06283}, EvoDS adopts an adaptive context compression strategy. Specifically, we equip the Manager Agent with a dedicated summarization tool $a^{\text{sum}} \in \mathcal{A}^{\text{man}}$ for dynamic context management. Given the current context $C$, the Manager Agent autonomously determines when to invoke the summarization tool according to the current reasoning state and overall task objective as $a^{\text{sum}} \sim \pi^{\text{man}}_\theta(\cdot \mid C, G)$, where $\pi^{\text{man}}_\theta$ denotes the Manager Agent. Upon invocation, the Manager Agent compresses the accumulated context and updates the Manager memory as $C \leftarrow g(C \mid G)$, where $g(\cdot)$ is a prompt-based summarization function that preserves task progress, key decisions, and critical intermediate results relevant to the global objective. This design enables adaptive rather than reactive context management, allowing the Manager Agent to selectively compress context at appropriate stages of execution.

By combining sub-agent-level abstraction with Manager-level adaptive compression, EvoDS effectively controls context growth in long-horizon data science workflows. This strategy mitigates token budget limitations, alleviates the \emph{lost-in-the-middle} effect~\cite{DBLP:journals/tacl/LiuLHPBPL24}, and supports coherent reasoning over complex multi-step tasks.

\subsection{Agentic Optimization for EvoDS}
\eat{
Training EvoDS poses several non-trivial challenges.
First, different agents serve distinct functional roles and optimize different objectives.
Second, the action space of each sub-agent evolves over time through adaptive tool synthesis, introducing non-stationarity into the policy learning process.
Third, the Manager Agent’s context compression strategy may induce context discontinuities.
To address these challenges, we adopt an online reinforcement learning framework that jointly optimizes task performance, action space evolution, and context management within a unified training process.

To reduce computational and memory overhead, following MATPO~\cite{DBLP:journals/corr/abs-2510-04678}, all agents share the same LLM backbone and parameters, while agent-specific behaviors are induced through role-specific system prompts and distinct action spaces.
However, directly applying reinforcement learning in this setting remains challenging due to reward sparsity.
Therefore, we follow a widely adopted two-stage training paradigm consisting of SFT warm-up followed by online RL.
}

To jointly optimize task execution, skill acquisition, and context management, we first perform SFT and then adopt online RL for multi-role agents. Following MATPO~\cite{DBLP:journals/corr/abs-2510-04678}, all agents in EvoDS share the same LLM backbone and parameters to reduce computational and memory overhead, while agent-specific behaviors are induced through role-specific system prompts and distinct action spaces.

\subsubsection{SFT for Stable Initialization}
We first employ an advanced LLM as a teacher model to collect trajectories for SFT. For each problem instance, the Manager Agent iteratively solves the overall problem, while sub-agents iteratively solve delegated subtasks. This process produces a \emph{main trajectory} $\tau^{\text{main}}$ from the Manager Agent and corresponding \emph{sub-trajectories} $\tau^{\text{sub}}$ from sub-agents:
\begin{align}
\tau^{\text{main}} &= \{ p, q, a_1, o_1, a_2, o_2, \dots, a_T, o_T \}, \\
\tau^{\text{sub}} &= \{ p_i, q_j, a^{\text{sub}}_1, o^{\text{sub}}_1, \dots, a^{\text{sub}}_S, o^{\text{sub}}_S \},
\end{align}
where $p$ and $p_i$ denote the system prompts of the Manager Agent and the $i$-th sub-agent, $q$ denotes the problem description, $q_j$ denotes the assigned subtask, and $(a, o)$ denotes the action–observation pairs. Both main and sub-trajectories are used for SFT to enable coordinated behavior initialization across agent roles.

Due to context summarization, the main trajectory may be partitioned into multiple segments. When a summarization action is invoked, the accumulated history is compressed, and subsequent decisions are conditioned on the summarized context. Accordingly, the main trajectory can be decomposed as
$
\tau^{\text{main}} = \bigcup_{k=1}^{K} \tau^{\text{main}}_k,
$
and each segment is treated as an independent trajectory for training.

\subsubsection{Joint RL for Multi-Role Agents}
After SFT warm-up, EvoDS is further optimized through online RL. To support joint optimization of multi-role agents, we define rewards separately for main trajectories and sub-trajectories.

For sub-trajectories, we adopt a rule-based reward reflecting whether the assigned subtask is successfully solved:
\begin{equation}
R^{\text{sub}} =
\begin{cases}
0.1, & \text{if the subtask is solved}, \\
-0.1, & \text{otherwise}.
\end{cases}
\end{equation}

For main trajectories, we design a hybrid reward balancing solution quality and coordination efficiency. Specifically, the outcome reward $R_{\text{outcome}} \in [0,1]$ measures the final solution quality. To encourage effective sub-agent scheduling, we define a subtask completion reward
$
R_{\text{sub}} = \frac{1}{N} \sum_{i=1}^{N} R^{\text{sub}}_i,
$
where $N$ is the number of sub-agent invocations. Efficiency is further encouraged via a context penalty $
P_{\text{context}} = |C|/C_{\max},
$
and a turn penalty
$
P_{\text{turn}} = T/T_{\max},
$
where $P_{\text{context}}$ measures the ratio of trajectory context length to the maximum token budget, while $P_{\text{turn}}$ captures the ratio of interaction turns to the maximum turn budget. The overall reward for the Manager Agent is defined as:
\begin{equation}
R^{\text{main}} = R_{\text{outcome}} + \alpha R_{\text{sub}} - \beta P_{\text{context}} - \gamma P_{\text{turn}},
\end{equation}
where we set $\alpha = 0.2$ and $\beta = \gamma = 0.1$ in practice. Since context summarization partitions a main trajectory into multiple segments while rewards are only available at the end of the rollout, we apply reward broadcasting by assigning the final reward to all trajectory segments as
$
R^{\text{main},k} \leftarrow R^{\text{main}}, \forall k \in \{1,\dots,K\}.
$

To jointly optimize the Manager Agent and sub-agents, we adopt Group Relative Policy Optimization (GRPO)~\cite{DBLP:journals/corr/abs-2402-03300}, which provides stable advantage estimation without requiring an explicit value function. For each problem instance, we generate a group of $n$ rollouts, resulting in $n$ main trajectories and $m$ sub-trajectories. For main trajectories, the advantage for the $i$-th rollout is computed as
$
A_i^{\text{main}} =
\frac{R_i^{\text{main}} - \mu(R^{\text{main}})}{\sigma(R^{\text{main}}) + \epsilon},
$
where $\mu(\cdot)$ and $\sigma(\cdot)$ denote the mean and standard deviation over the rollout group, and $\epsilon$ is a small constant. For sub-trajectories, given the binary reward structure, we directly use
$
A^{\text{sub}} = R^{\text{sub}}.
$

The overall training objective jointly optimizes the policy over both trajectory types.
We define the clipped surrogate objective per trajectory as $\mathcal{L}_{\text{clip}}(\tau, A)$.
The total loss is given by:
{\tiny
\begin{align}
\mathcal{L}(\theta) =
\mathbb{E} \Bigg[
\frac{1}{n} \sum_{i=1}^{n} \mathcal{L}_{\text{clip}}(\tau^{\text{main}}_i, A_i^{\text{main}})
+
\frac{1}{m} \sum_{j=1}^{m} \mathcal{L}_{\text{clip}}(\tau^{\text{sub}}_j, A_j^{\text{sub}})
\Bigg]
-
\beta_{\text{KL}} \, \mathbb{D}_{\text{KL}}(\pi_\theta \,\|\, \pi_{\text{ref}}),
\\
\mathcal{L}_{\text{clip}}(\tau, A) =
\frac{1}{|\tau|} \sum_{t=1}^{|\tau|}
\min \left(
\frac{\pi_\theta(a_t|s_t)}{\pi_{\theta_{\text{old}}}(a_t|s_t)} A,\,
\text{clip}\!\left(
\frac{\pi_\theta(a_t|s_t)}{\pi_{\theta_{\text{old}}}(a_t|s_t)},
1-\varepsilon,
1+\varepsilon
\right) A
\right),
\end{align}}
where $\varepsilon$ and $\beta_{\text{KL}}$ are hyperparameters, $\pi_{\text{ref}}$ denotes a reference policy, and $\mathbb{D}_{\text{KL}}$ denotes the Kullback--Leibler divergence.

Through this unified training strategy, EvoDS jointly learns high-level task orchestration and low-level subtask execution. The combination of trajectory segmentation, reward broadcasting, and GRPO-based optimization enables robust long-horizon reasoning with adaptive skill acquisition and effective context management.

\section{Theoretical Analysis}
\label{sec:theorem}

In this section, we provide theoretical insights into several key design choices of EvoDS. Since skills are represented as executable tools, we begin by analyzing the advantages of the proposed hierarchical agent framework for tool selection.

\begin{theorem}
\label{theorem_1}
Given a fixed context $C$, the upper bound on the tool selection error probability of a hierarchical agent framework is strictly lower than that of a flat agent framework.
\end{theorem}
The proof of Theorem~\ref{theorem_1} is provided in Appendix~\ref{appendix: error_bound}. This result shows that hierarchical tool selection yields a tighter error bound by decomposing a large decision space into smaller, structured sub-problems, thereby reducing uncertainty and improving selection robustness. We next analyze the Manager Agent’s optimization objective from an information-theoretic perspective.
\begin{theorem}
\label{theorem_2}
The optimization objective of the Manager Agent is equivalent to solving the following Information Bottleneck problem:
\begin{equation}
    \min_{p(z \mid c)} I(Z; C) - \lambda I(Z; Y), \quad \lambda > 0,
\end{equation}
where $C$ denotes the global context, $Z$ denotes the compressed context, and $Y$ denotes the final agent output.
\end{theorem}
The proof of Theorem~\ref{theorem_2} is provided in Appendix~\ref{appendix: objective}. The theorem shows that the Manager Agent implicitly performs Information Bottleneck optimization, which minimizes task-irrelevant information while preserving task-critical signals. This provides a theoretical justification for our optimization objective design, showing that the proposed context compression strategy can balance efficiency and decision quality.

\begin{table*}[!t]
\caption{Performance comparison of different agents. The best proprietary and open-source model results are highlighted in bold, respectively. EvoDS disables synthesized skill reuse, while EvoDS-evo enables cross-task reuse within the same benchmark.}
\label{tab:compare}
\begin{tabular}{ccccccc}
\toprule
\multicolumn{1}{c|}{Backbones}                       & \multicolumn{1}{c|}{Methods}          & DABench        & DA-Code         & ScienceAgentBench & \multicolumn{1}{c|}{MLE-Dojo}        & AVG.           \\ \hline
\multicolumn{7}{c}{\cellcolor[HTML]{E6F0FA}Proprietary Models}                                                                                                                                            \\ \hline
\multicolumn{1}{c|}{}                                & \multicolumn{1}{c|}{AutoGen}          & 0.756          & 0.260           & 0.098             & \multicolumn{1}{c|}{0.137}          & 0.313          \\
\multicolumn{1}{c|}{}                                & \multicolumn{1}{c|}{ReAct}            & 0.858          & 0.367          & 0.206             & \multicolumn{1}{c|}{0.199}          & 0.408          \\
\multicolumn{1}{c|}{}                                & \multicolumn{1}{c|}{LAMBDA}           & 0.859          & 0.215          & 0.108             & \multicolumn{1}{c|}{0.265}          & 0.362          \\
\multicolumn{1}{c|}{\multirow{-2}{*}{DeepSeek-V3.1}} & \multicolumn{1}{c|}{Data Interpreter} & 0.746          & 0.234          & 0.088             & \multicolumn{1}{c|}{0.126}          & 0.299          \\ 
\multicolumn{1}{c|}{}                                & \multicolumn{1}{c|}{LATM}   & 0.769 & 0.208 & 0.186             & \multicolumn{1}{c|}{0.140}          & 0.327                  \\
\multicolumn{1}{c|}{}                                & \multicolumn{1}{c|}{ML-Master2}           & 0.884 & 0.276 & 0.157            & \multicolumn{1}{c|}{0.299}          & 0.404          \\ \hline
\multicolumn{1}{c|}{}                                & \multicolumn{1}{c|}{AutoGen}          & 0.771          & 0.192          & 0.088             & \multicolumn{1}{c|}{0.108}          & 0.290           \\
\multicolumn{1}{c|}{}                                & \multicolumn{1}{c|}{Code Interpreter} & \textbf{0.924} & 0.171          & 0.088             & \multicolumn{1}{c|}{0.181}          & 0.341          \\
\multicolumn{1}{c|}{}                                & \multicolumn{1}{c|}{ReAct}            & 0.781          & 0.303          & 0.147             & \multicolumn{1}{c|}{0.228}          & 0.365          \\
\multicolumn{1}{c|}{}                                & \multicolumn{1}{c|}{LAMBDA}           & 0.881          & 0.200            & 0.118             & \multicolumn{1}{c|}{0.268}          & 0.367          \\
\multicolumn{1}{c|}{\multirow{-3}{*}{GPT-4o}}        & \multicolumn{1}{c|}{Data Interpreter} & 0.716          & 0.244          & 0.088             & \multicolumn{1}{c|}{0.102}          & 0.288          \\ 
\multicolumn{1}{c|}{}                                & \multicolumn{1}{c|}{LATM}           & 0.746 & 0.146 & 0.108             & \multicolumn{1}{c|}{0.132}          & 0.283          \\
\multicolumn{1}{c|}{}                                & \multicolumn{1}{c|}{ML-Master2}           & 0.874 & 0.195 & 0.088             & \multicolumn{1}{c|}{0.273}          & 0.358          \\ \hline
\multicolumn{1}{c|}{}                                & \multicolumn{1}{c|}{AutoGen}          & 0.884          & 0.274          & 0.108             & \multicolumn{1}{c|}{0.165}          & 0.358          \\
\multicolumn{1}{c|}{}                                & \multicolumn{1}{c|}{ReAct}            & 0.866          & \textbf{0.399} & \textbf{0.225}    & \multicolumn{1}{c|}{\textbf{0.308}} & \textbf{0.450}  \\
\multicolumn{1}{c|}{}                                & \multicolumn{1}{c|}{LAMBDA}           & 0.919          & 0.221          & 0.127             & \multicolumn{1}{c|}{0.254}          & 0.380           \\
\multicolumn{1}{c|}{\multirow{-2}{*}{o4-mini}}       & \multicolumn{1}{c|}{Data Interpreter} & 0.788          & 0.227          & 0.098             & \multicolumn{1}{c|}{0.152}          & 0.316          \\ 
\multicolumn{1}{c|}{}                                & \multicolumn{1}{c|}{LATM}           & 0.744 & 0.215 & \textbf{0.225}             & \multicolumn{1}{c|}{0.173}          & 0.339          \\
\multicolumn{1}{c|}{}                                & \multicolumn{1}{c|}{ML-Master2}           & 0.884 & 0.335 & 0.176             & \multicolumn{1}{c|}{0.298}          & 0.423          \\ \hline
\multicolumn{7}{c}{\cellcolor[HTML]{F3F8FD}Open-source Models}                                                                                                                                            \\ \hline
\multicolumn{1}{c|}{DeepSeek-R1-0528-Qwen3-8B}                                & \multicolumn{1}{c|}{DeepAnalyze-8B}   & 0.856           & 0.217          & 0.000              & \multicolumn{1}{c|}{0.044}          & 0.279          \\
\multicolumn{1}{c|}{Qwen2.5-Coder 7B}                                & \multicolumn{1}{c|}{DataMind-7B}      & 0.829           & 0.211          & 0.010             & \multicolumn{1}{c|}{0.010}         & 0.265          \\
\multicolumn{1}{c|}{Qwen2.5-Coder 14B}                                & \multicolumn{1}{c|}{DataMind-14B}     & 0.876          & 0.292          & 0.010             & \multicolumn{1}{c|}{0.136}          & 0.329           \\
\multicolumn{1}{c|}{Qwen3-8B}             & \multicolumn{1}{c|}{\textbf{EvoDS-8B (Ours)}}         & 0.894 & 0.337 & 0.108    & \multicolumn{1}{c|}{0.302} & 0.410 \\
\multicolumn{1}{c|}{Qwen3-8B}             & \multicolumn{1}{c|}{\textbf{EvoDS-evo-8B (Ours)}}         & \textbf{0.911} & \textbf{0.355} & \textbf{0.118}    & \multicolumn{1}{c|}{\textbf{0.311}} & \textbf{0.424} \\ \bottomrule
\end{tabular}
\end{table*}

\section{Experiments}
\eat{
In this section, we conduct experiments to answer the following research questions:
\begin{itemize}
    \item \textbf{RQ1}: How does EvoDS perform compared with state-of-the-art agents?
    \item \textbf{RQ2}: How do different modules contribute to the overall effectiveness of EvoDS?
    \item \textbf{RQ3}: Can the proposed Adaptive Tool Evolution Mechanism generate effective and reusable tools in practice?
\end{itemize}
}

This section aims to answer the following research questions: 
\begin{itemize} 
\item \textbf{RQ1}: How effective is EvoDS compared with state-of-the-art autonomous data science agents? 
\item \textbf{RQ2}: What is the contribution of each module to EvoDS? 
\item \textbf{RQ3}: Can EvoDS synthesize effective and reusable skills? 
\textbf{RQ4}: What are the successful and failure cases of EvoDS? \end{itemize}

\subsection{Experimental Setup}

\subsubsection{Benchmarks}

\eat{We evaluate EvoDS on four representative data science benchmarks, including DABench~\cite{DBLP:conf/icml/HuZWCM0WSXZCY0K24}, DA-Code~\cite{DBLP:conf/emnlp/HuangLYZLWHHLZL24}, ScienceAgentBench (SAB)~\cite{DBLP:conf/iclr/ChenCNZWYLLWLDX25}, and MLE-Dojo~\cite{qiang2025mledojo}, to comprehensively assess its effectiveness across diverse task settings. These benchmarks cover a broad spectrum of data science tasks, such as data wrangling, exploratory data analysis, and predictive data modeling.

For visualization tasks in DA-Code, the benchmark evaluates generated executable code, whereas EvoDS is a tool-integrated agent that directly produces plots via visualization tools. Therefore, following MatPlotBench~\cite{DBLP:conf/acl/YangZWCHYLTLYLS24}, we adopt an LLM-as-a-judge strategy to directly evaluate the generated plots. In addition, due to the high computational cost of machine learning tasks, we sample 10 problem instances from MLE-Dojo for evaluation. Detailed descriptions of all benchmarks are provided in the appendix.}

We evaluate EvoDS on four data science benchmarks, including DABench~\cite{DBLP:conf/icml/HuZWCM0WSXZCY0K24}, DA-Code~\cite{DBLP:conf/emnlp/HuangLYZLWHHLZL24}, ScienceAgentBench (SAB)~\cite{DBLP:conf/iclr/ChenCNZWYLLWLDX25}, and MLE-Dojo~\cite{qiang2025mledojo}, to assess its effectiveness across diverse tasks including data wrangling, exploratory analysis, and predictive modeling. For visualization tasks in DA-Code, since EvoDS directly generates plots via tools rather than executable code, we adopt an LLM-as-a-judge strategy following MatPlotBench~\cite{DBLP:conf/acl/YangZWCHYLTLYLS24}. In addition, due to the high computational cost of machine learning tasks, we evaluate EvoDS on 10 sampled instances from MLE-Dojo. Detailed benchmark descriptions are provided in the Appendix~\ref{appendix: bench}.

\subsubsection{Baselines}

We compare EvoDS against a diverse set of competitive baselines, including general-purpose agents such as AutoGen~\cite{wu2024autogen}, ReAct~\cite{DBLP:conf/iclr/YaoZYDSN023}, and Code Interpreter~\cite{openai2023}, as well as data science agents including LAMBDA~\cite{Sun17072025}, Data Interpreter~\cite{DBLP:conf/acl/HongLLLWZLCZWZZ25}, DataMind (7B and 14B)~\cite{DBLP:journals/corr/abs-2509-25084}, DeepAnalyze-8B~\cite{DBLP:journals/corr/abs-2510-16872}, and self-evolving agents LATM~\cite{DBLP:conf/iclr/Cai00CZ24} and ML-Master2~\cite{DBLP:journals/corr/abs-2601-10402}. Except for DataMind and DeepAnalyze, which are trained agents with fixed backbones, all other agents are evaluated using three different LLM backbones: DeepSeek-V3.1-Terminus~\cite{deepseek2025v31}, GPT-4o~\cite{openai2023gpt4}, and o4-mini~\cite{openai2025o3o4mini}. 
\eat{Detailed baseline descriptions are provided in the Appendix~\ref{appendix: baselines}.}

\subsubsection{Implementation Details}

EvoDS uses Qwen3-8B~\cite{DBLP:journals/corr/abs-2505-09388} as the shared backbone for all agents. Training data are constructed from heterogeneous sources, including \textit{DataMind-12K}~\cite{DBLP:journals/corr/abs-2509-25084}, \textit{DataScience-Instruct-500K}~\cite{DBLP:journals/corr/abs-2510-16872}, \textit{MatPlotBench}~\cite{DBLP:conf/acl/YangZWCHYLTLYLS24}, \textit{DSBench}~\cite{DBLP:conf/iclr/JingHWYYM0DY25}, and \textit{MLE-Dojo}~\cite{qiang2025mledojo}, covering data analysis, visualization, and machine learning tasks. Since DataMind-12K and DataScience-Instruct-500K are trajectory-based datasets, we use GPT-4o to extract problem descriptions and corresponding ground-truth answers. Due to the high cost of machine learning tasks, only a subset of instances from DSBench and MLE-Dojo is sampled for training, with no overlap with the test set. We further use Qwen3-8B to filter overly simple samples, resulting in 8K training instances.

For SFT, we use the EvoDS framework with DeepSeek-V3.1 as the teacher model, collecting 8 rollouts per instance and yielding 36K trajectories in total. We train for 3 epochs with a batch size of 32 and a learning rate of $1\times10^{-5}$. For RL, we apply curriculum learning by gradually increasing the interaction turn budget from 4 to 20, using a rollout size of 8 and a learning rate of $1\times10^{-6}$ for 300 steps. The maximum response length is set to 24K tokens. All training is conducted using VeRL~\cite{DBLP:journals/corr/abs-2509-01055} on 4 NVIDIA A800 GPUs.

\subsection{Performance Comparison (RQ1)}

Table~\ref{tab:compare} reports the performance of EvoDS on four data science benchmarks. Overall, EvoDS achieves consistently strong performance and establishes a new state-of-the-art among open-source backbones across all benchmarks. Moreover, EvoDS-evo consistently outperforms EvoDS, indicating that the proposed Autonomous Skill Acquisition mechanism effectively enhances agent skills and leads to improved performance. Specifically, compared with the strongest open-source baseline, DataMind-14B, EvoDS achieves an absolute improvement of 9.5\% and a relative improvement of 28.9\% in average performance, despite using fewer model parameters. These results indicate that the performance gains are not driven by model scale, but rather by the proposed agent framework and training strategy, demonstrating its effectiveness in handling diverse data science scenarios.

EvoDS also demonstrates competitive performance against proprietary methods. Overall, EvoDS achieves the second-best average performance, surpassed only by ReAct with the o4-mini backbone, while outperforming all other proprietary baselines. Notably, EvoDS exceeds the best DeepSeek-V3.1-based baseline by 3.9\% on average and outperforms the strongest GPT-4o-based baseline by 15.5\%. These results suggest that a carefully designed agent framework combined with effective training can substantially narrow the performance gap between open-source and proprietary foundation models. Although DeepSeek-V3.1 is used as the teacher model during SFT, EvoDS achieves superior inference-time performance, highlighting the effectiveness of reinforcement learning in improving long-horizon decision making beyond imitation.

Beyond overall performance, EvoDS demonstrates strong long-horizon and end-to-end data science capabilities. While DABench focuses on relatively simple data analysis tasks, DA-Code, ScienceAgentBench, and MLE-Dojo emphasize iterative interaction, long-horizon reasoning, and end-to-end workflows. EvoDS consistently outperforms open-source baselines on these benchmarks, with especially large gains on DA-Code and MLE-Dojo. On MLE-Dojo, where all other open-source baselines perform poorly (with the best result reaching only 0.136), EvoDS achieves a score of 0.311 and even surpasses proprietary baselines such as ReAct (o4-mini). These results further highlight EvoDS’s ability to handle long-horizon tasks and autonomously coordinate the full data science pipeline, rather than excelling only at isolated subtasks.

Despite its strong overall performance, EvoDS remains challenged on the most difficult data-driven scientific discovery tasks in ScienceAgentBench, where a performance gap persists compared with the strongest proprietary baselines. These tasks typically require deep domain knowledge and abstract reasoning beyond procedural tool usage. While EvoDS benefits from its agent framework and training strategy, its performance is ultimately constrained by the scientific knowledge of the underlying foundation model. Incorporating stronger domain knowledge or external knowledge sources remains an important direction for future work.

\subsection{Ablation Studies (RQ2)}
\begin{table}[!t]
\caption{Ablation study results.}
\label{table: ablation}
\resizebox{\linewidth}{!}{
\begin{tabular}{c|cccc|c}
\toprule
Methods   & DABench        & DA-Code         & SAB & MLE-Dojo        & AVG.           \\ \midrule
w/o train & 0.673          & 0.101          & 0.020             & 0.000          & 0.199          \\
w/o rl    & 0.878          & 0.325          & 0.059             & 0.188          & 0.363          \\
w/ grpo    & 0.869          & 0.333          & 0.108             & 0.287          & 0.399          \\
w/o tool  & 0.882          & 0.336          & 0.059             & 0.184          & 0.365          \\
w/o hier  & 0.900          & 0.337          & 0.069             & 0.251          & 0.389          \\
w/o asa   & 0.894          & 0.352          & 0.088             & 0.257          & 0.398          \\
w/o acc   & 0.865          & 0.333          & 0.098             & 0.122          & 0.355          \\ \hline
EvoDS     & \textbf{0.911} & \textbf{0.355} & \textbf{0.118}    & \textbf{0.311} & \textbf{0.424} \\ \bottomrule
\end{tabular}}
\end{table}

To assess the contribution of each module in EvoDS, we conduct ablation studies with seven variants, each removing or modifying a specific module. Specifically, \textbf{"w/o train"} directly uses Qwen3-8B without training; \textbf{"w/o rl"} applies SFT only; \textbf{"w/ grpo"} trains only the Manager Agent using GRPO, without the proposed RL strategy; \textbf{"w/o tool"} restricts the agent to a single code execution tool; \textbf{"w/o hier"} replaces the hierarchical agent architecture with a flat architecture; \textbf{"w/o asa"} disables the Autonomous Skill Acquisition mechanism; and \textbf{"w/o acc"} removes the Adaptive Context Compression strategy. Results are reported in Table~\ref{table: ablation}. Overall, EvoDS consistently outperforms all ablated variants across all benchmarks, indicating that each proposed module contributes positively to performance. The large performance gap between "w/o train" and "w/o rl" highlights the importance of SFT for initializing agent behavior. Moreover, "w/ grpo" outperforms "w/o rl", while EvoDS further improves upon "w/ grpo", demonstrating that RL plays a critical role in refining long-horizon decision making, tool coordination, and execution beyond imitation, and further validating the effectiveness of the proposed joint reinforcement learning strategy for multi-role agents. Comparing "w/o tool" and "w/o hier", we observe that "w/o hier" consistently achieves better performance, indicating the importance of integrating explicit data science tools rather than relying solely on code execution. However, EvoDS further outperforms "w/o hier", demonstrating that the hierarchical agent framework provides additional benefits. This result empirically supports Theorem~\ref{theorem_1}, which shows that hierarchical tool selection yields a tighter error bound than flat selection by decomposing a large decision space into structured sub-problems.

We further observe that removing either the Autonomous Skill Acquisition Mechanism ("w/o asa") or the Adaptive Context Compression strategy ("w/o acc") leads to noticeable performance degradation, indicating that dynamically synthesizing and reusing synthesized tools improves agent capability. Meanwhile, "w/o acc" exhibits the most severe degradation among all trained variants, suggesting that effective context management is essential for stable long-horizon execution. To better understand the impact of context compression, we report the proportion of test instances affected by out-of-token-limit failures in Table~\ref{table: oom}. Without context compression, the agent frequently exceeds the token budget, especially on complex benchmarks such as ScienceAgentBench and MLE-Dojo that require long-horizon reasoning and extensive intermediate context. In contrast, EvoDS eliminates out-of-token-limit failures entirely across all benchmarks, demonstrating that the proposed Adaptive Context Compression strategy is crucial for scalability and reliability in long-horizon data science tasks.

\begin{table}[!t]
\caption{Proportion of samples exceeding token limits.}
\label{table: oom}
\begin{tabular}{c|cccc}
\toprule
\multicolumn{1}{l|}{Methods} & DABench & DACode & SAB    & MLEDojo \\ \midrule
w/o acc                      & 7/257   & 20/500 & 18/102 & 3/10    \\
EvoDS                        & 0/257   & 0/500  & 0/102  & 0/10    \\ \bottomrule
\end{tabular}
\end{table}

\begin{figure*}[!t]
  \centerline{\includegraphics[width=1\linewidth]{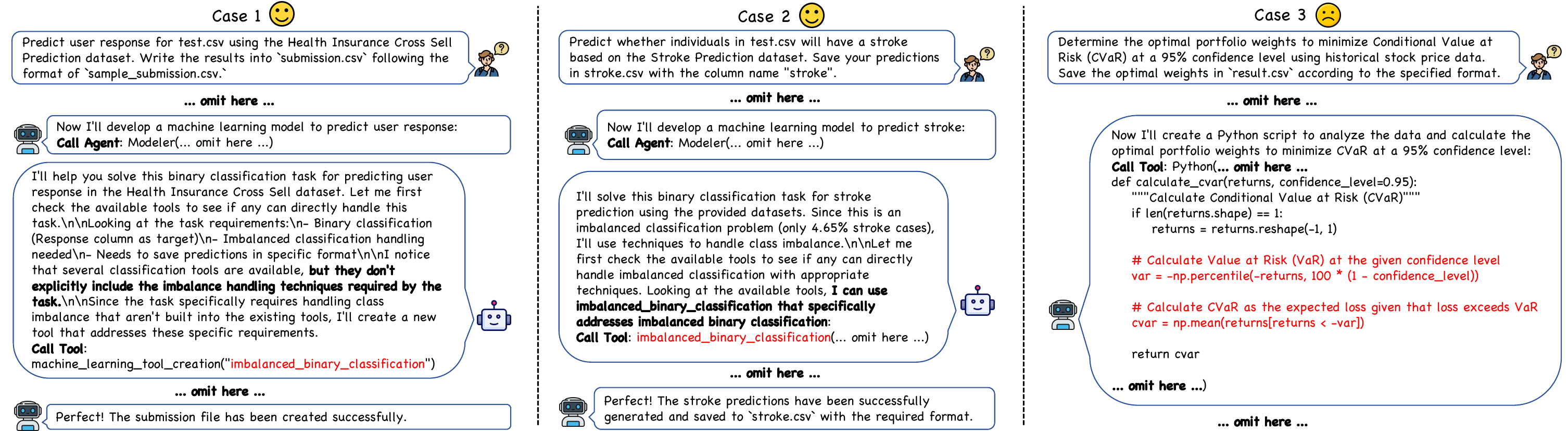}}
  \caption{Case study of EvoDS. Case 1 shows skill synthesis for solving a new task. Case 2 demonstrates cross-task skill reuse. Case 3 presents a failure case on a professional quantitative finance task due to insufficient domain expertise.}
  \label{fig:cases}
\end{figure*}

\subsection{Cross-Task Skill Reuse (RQ3)}

To evaluate whether the Autonomous Skill Acquisition mechanism enables EvoDS to acquire and reuse problem-solving skills across tasks, we conduct experiments on both within-benchmark and cross-benchmark cross-task skill reuse and generalization. For the within-benchmark setting, we split the DA-Code benchmark into a validation set and a test set with a ratio of 3:2. We compare two variants: \textbf{"w/o reuse"}, which evaluates EvoDS on the test set without retaining any previously synthesized skills, and \textbf{"w/ reuse"}, which first evaluates EvoDS on the validation set, retains the synthesized skills, and then evaluates on the test set with access to these accumulated skills. For the cross-benchmark setting, we conduct an additional experiment on ScienceAgentBench. In this setting, "w/ reuse" leverages skills collected from the DA-Code validation set, while "w/o reuse" is evaluated directly on ScienceAgentBench 

The results are reported in Table~\ref{table:evo}. As shown, "w/ reuse" consistently outperforms "w/o reuse", achieving improvements of 2.9\% and 9.3\% on the challenging DA-Code and SAB benchmarks, respectively. These results demonstrate that synthesized skills from earlier tasks can effectively facilitate subsequent problem solving. Notably, the performance gain on ScienceAgentBench indicates that the synthesized skills remain beneficial under substantial distribution shifts, suggesting that they capture reusable capabilities rather than task-specific procedures. To further analyze the effectiveness of skill reuse, we additionally collect statistics of synthesized skills during testing. EvoDS synthesized 279 skills in total, which were invoked 925 times with a 69\% cross-task reuse rate, further demonstrating the strong generalizability and reusability of the acquired skills across diverse tasks.

\begin{table}[t!]
\centering
\caption{Performance comparison with and without the reuse of synthesized skills.}
\begin{tabular}{c|cc}
\toprule
Methods & DA-Code & ScienceAgentBench \\ \midrule
w/o reuse & 0.341 & 0.108 \\
w/ reuse  & 0.351 & 0.118 \\ 
\bottomrule
\end{tabular}
\label{table:evo}
\end{table}

\subsection{Case Studies and Failure Analysis (RQ4)}
In this section, we present case studies to illustrate the self-evolving behavior, limitations, and failure modes of EvoDS. We select three representative tasks from DA-Code and omit intermediate details for clarity. As shown in Figure~\ref{fig:cases}, in Case~1, EvoDS identifies that the task requires handling severe class imbalance, while the predefined skill set lacks a suitable solution. EvoDS therefore synthesizes an \textit{imbalanced\_binary\_classification} skill tailored to the task. In Case~2, EvoDS encounters another task with similar imbalance characteristics and directly reuses the previously synthesized skill, avoiding redundant exploration and enabling more efficient problem solving. These examples demonstrate that EvoDS can not only solve complex data science tasks through adaptive skill synthesis, but also accumulate reusable capabilities across related tasks.

In contrast, Case~3 presents a failure scenario on a professional quantitative finance task requiring expertise in financial modeling and numerical optimization. Due to insufficient domain knowledge, EvoDS produces flawed execution logic, leading to task failure. To further analyze the limitations of EvoDS, we investigate 50 failed cases on DA-Code and categorize them into four major types: (1) Instruction Following Errors (52\%), where the agent fails to follow task instructions or constraints; (2) Execution Limits (18\%), where complex tasks cannot be solved within the interaction budget; (3) Coordination Errors (18\%), where ineffective coordination among agents causes information loss or execution failures; and (4) Reasoning Deficits (12\%), where the execution logic is flawed. These results suggest that, although EvoDS demonstrates strong adaptability and skill reuse ability, further improvements are still needed in domain expertise, long-horizon coordination, and robust reasoning.

\begin{figure}[!t]
  \centerline{\includegraphics[width=0.98\linewidth]{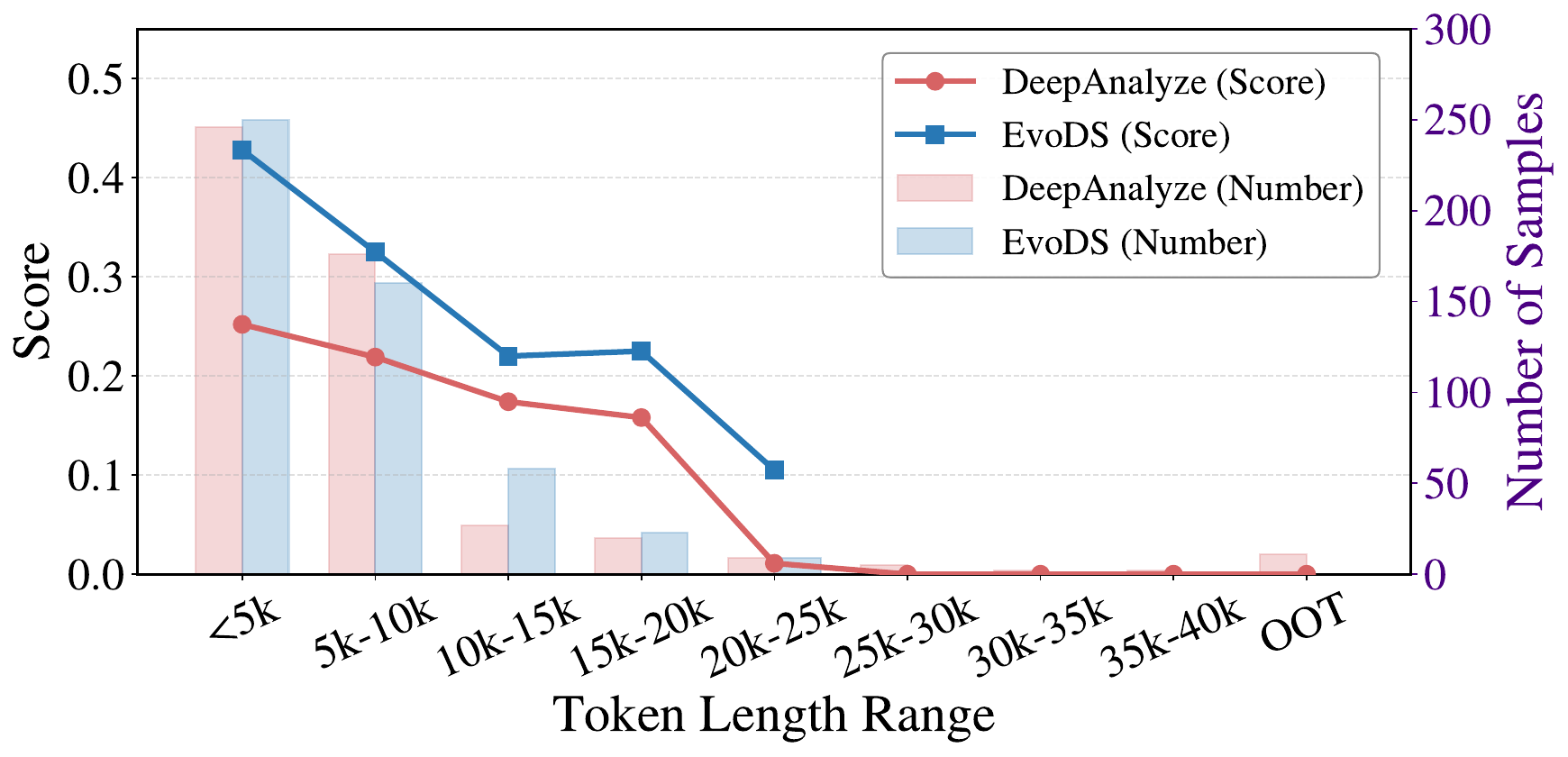}}
  \caption{Performance of EvoDS and DeepAnalyze across varying context token length ranges on DA-Code, where OOT denotes out-of-token limitations.}
  \label{fig:context}
\end{figure}

\subsection{Further Analysis}

To examine EvoDS’s effectiveness on complex long-horizon tasks, we compare EvoDS with DeepAnalyze on DA-Code under varying context lengths, as shown in Figure~\ref{fig:context}. As the context length increases, the performance of DeepAnalyze degrades sharply, failing to solve most instances with contexts exceeding 20k tokens and frequently encountering out-of-token-limit failures. In contrast, although EvoDS also exhibits a performance decline, it consistently maintains higher accuracy across all context ranges and still achieves around 10\% accuracy on instances with context lengths between 20k and 25k. Notably, no EvoDS instances exceed 25k tokens. These results indicate that the proposed Adaptive Context Compression strategy effectively controls context growth, enabling more stable and scalable long-horizon execution.

We further analyze the proposed joint reinforcement learning strategy for multi-role agents by comparing it with a naive GRPO baseline that optimizes only the Manager Agent using main trajectories. Figure~\ref{fig:training} presents the average validation reward and context token usage during training for both methods. EvoDS consistently achieves higher validation performance while using fewer context tokens throughout training. These results indicate that jointly optimizing the Manager and sub-agent behaviors leads to more effective coordination, which is essential for learning in multi-agent settings. During training, EvoDS first exhibits an increase in context usage due to the gradual expansion of the interaction turn budget, followed by a steady reduction driven by the penalty term introduced in the reward function to encourage more efficient execution behaviors and compact context usage. Overall, these analyses demonstrate that EvoDS’s advantages on long-horizon tasks stem from both effective context management and principled agentic reinforcement learning.

\begin{figure}[!t]
  \centering
  \begin{subfigure}{0.23\textwidth}
    \includegraphics[width=\linewidth]{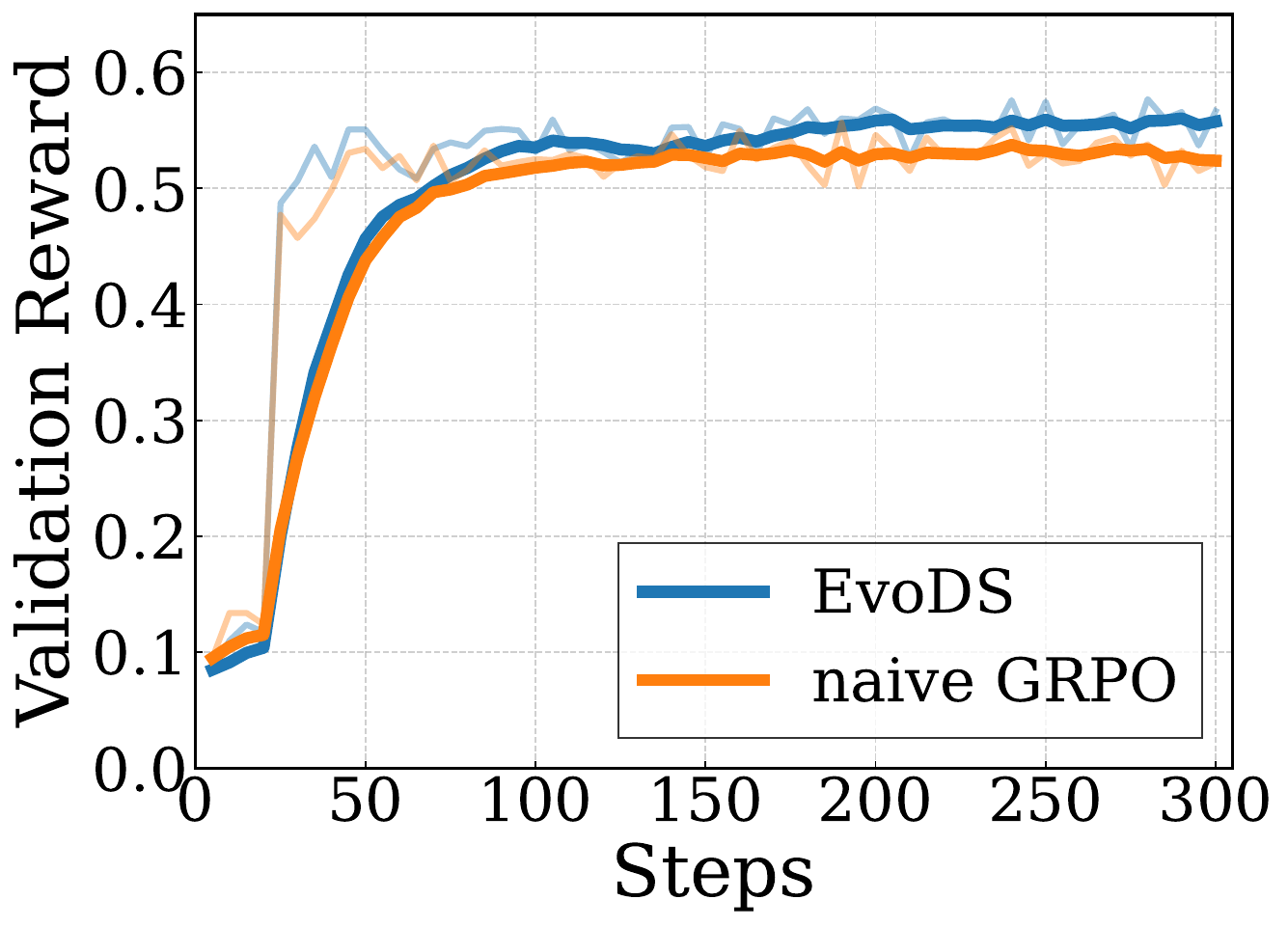}
    \caption{Reward over Training Steps.
    }
    \label{fig:rl}
  \end{subfigure}
  \begin{subfigure}{0.23\textwidth}
    \includegraphics[width=\linewidth]{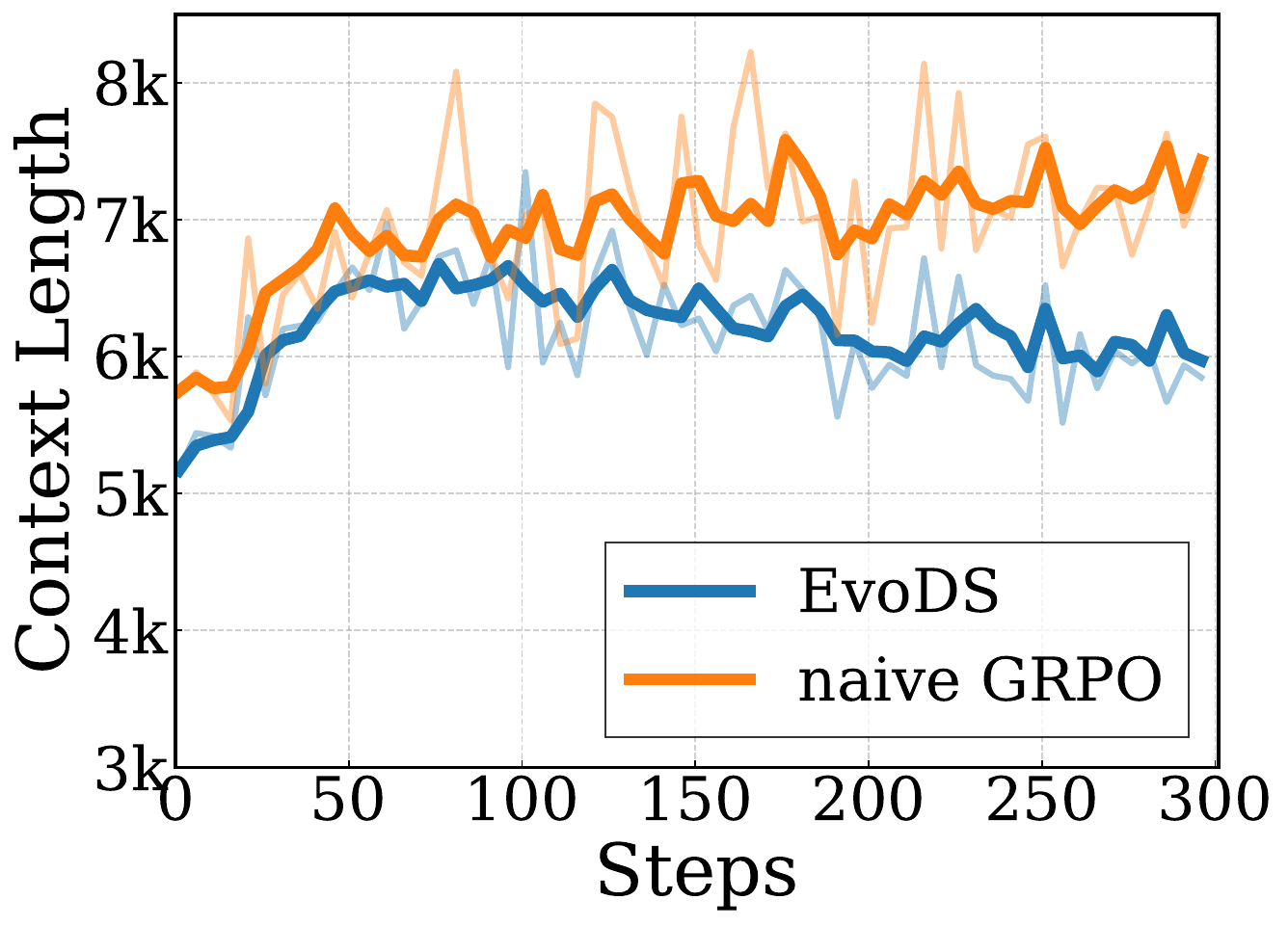}
    \caption{Context over Training Steps.}
    \label{fig:context_curve}
  \end{subfigure}
  \caption{Reward and context length of EvoDS and naive GRPO across RL training steps.}
  \label{fig:training}
\end{figure}

\section{Conclusion}
In this work, we presented EvoDS, a self-evolving autonomous data science agent designed to address key limitations of existing LLM-based data science systems, particularly their inability to acquire reusable skills from experience and effectively manage long-horizon execution contexts. By integrating a hierarchical multi-agent architecture with autonomous skill acquisition, adaptive context compression, and a joint reinforcement learning strategy for multi-role agents, EvoDS can continuously accumulate reusable skills, efficiently regulate long execution contexts, and solve complex data science tasks in an end-to-end manner. Extensive experiments and detailed analyses validate the effectiveness of each proposed component and demonstrate that EvoDS generalizes well across tasks and benchmarks, consistently outperforming both open-source and proprietary baselines. These findings suggest that self-evolving agent frameworks provide a promising direction toward scalable, adaptive, and autonomous data science systems.

\eat{
In this work, we presented EvoDS, a self-evolving data science agent designed to overcome key limitations of existing LLM-based data science agents, which lack the ability to acquire skills from experience and to effectively manage long-horizon execution contexts. By integrating a hierarchical multi-agent architecture with autonomous skill acquisition, context compression, and joint reinforcement learning strategy for multi-role agents, EvoDS continuously accumulates reusable skills, efficiently manages long contexts, and solves complex data science tasks in an end-to-end manner. Extensive experiments and analyses validate the contribution of each component and demonstrate that EvoDS generalizes well across tasks and benchmarks, achieving superior performance over both open-source and proprietary baselines. These results suggest that self-evolving agent frameworks offer a promising direction for building scalable and autonomous data science systems.}

\begin{acks}
This work was supported by the National Natural Science Foundation of China (Grant No. 62572417, No.92370204), National Key R\&D Program of China (Grant No.2023YFF0725004).
\end{acks}

\clearpage
\bibliographystyle{ACM-Reference-Format}
\bibliography{ref}

@inproceedings{DBLP:conf/iclr/AlemiFD017,
  author       = {Alexander A. Alemi and
                  Ian Fischer and
                  Joshua V. Dillon and
                  Kevin Murphy},
  title        = {Deep Variational Information Bottleneck},
  booktitle    = {{ICLR}},
  year         = {2017}
}

@article{DBLP:journals/tacl/LiuLHPBPL24,
  author       = {Nelson F. Liu and
                  Kevin Lin and
                  John Hewitt and
                  Ashwin Paranjape and
                  Michele Bevilacqua and
                  Fabio Petroni and
                  Percy Liang},
  title        = {Lost in the Middle: How Language Models Use Long Contexts},
  journal      = {Trans. Assoc. Comput. Linguistics},
  volume       = {12},
  pages        = {157--173},
  year         = {2024}
}

@inproceedings{
wu2024autogen,
title={AutoGen: Enabling Next-Gen {LLM} Applications via Multi-Agent Conversations},
author={Qingyun Wu and Gagan Bansal and Jieyu Zhang and Yiran Wu and Beibin Li and Erkang Zhu and Li Jiang and Xiaoyun Zhang and Shaokun Zhang and Jiale Liu and Ahmed Hassan Awadallah and Ryen W White and Doug Burger and Chi Wang},
booktitle={{COLM}},
year={2024}
}

@inproceedings{DBLP:conf/iclr/YaoZYDSN023,
  author       = {Shunyu Yao and
                  Jeffrey Zhao and
                  Dian Yu and
                  Nan Du and
                  Izhak Shafran and
                  Karthik R. Narasimhan and
                  Yuan Cao},
  title        = {ReAct: Synergizing Reasoning and Acting in Language Models},
  booktitle    = {{ICLR}},
  year         = {2023}
}

@inproceedings{DBLP:conf/icml/GuoD0C0024,
  author       = {Siyuan Guo and
                  Cheng Deng and
                  Ying Wen and
                  Hechang Chen and
                  Yi Chang and
                  Jun Wang},
  title        = {DS-Agent: Automated Data Science by Empowering Large Language Models with Case-Based Reasoning},
  booktitle    = {{ICML}},
  pages        = {16813--16848},
  year         = {2024}
}

@inproceedings{DBLP:conf/icml/TriratJH25,
  author       = {Patara Trirat and
                  Wonyong Jeong and
                  Sung Ju Hwang},
  title        = {AutoML-Agent: {A} Multi-Agent {LLM} Framework for Full-Pipeline AutoML},
  booktitle    = {{ICML}},
  pages = 	 {60099--60146},
  year         = {2025}
}

@inproceedings{
li2025autokaggle,
title={AutoKaggle: A Multi-Agent Framework for Autonomous Data Science Competitions},
author={Ziming Li and Qianbo Zang and David Ma and Jiawei Guo and Tianyu Zheng and Minghao Liu and Xinyao Niu and Yue Wang and Jian Yang and Jiaheng Liu and Wanjun Zhong and Wangchunshu Zhou and Stephen Huang and Ge Zhang},
booktitle={DL4C@ICLR},
year={2025}
}

@inproceedings{
liu2025mmagent,
title={{MM}-Agent: {LLM} as Agents for Real-world Mathematical Modeling Problem},
author={Fan Liu and Zhe-Rui Yang and Cancheng Liu and Tianrui SONG and Xiaofeng Gao and Hao Liu},
booktitle={NeurIPS},
year={2025}
}

@article{DBLP:journals/corr/abs-2510-16872,
  author       = {Shaolei Zhang and
                  Ju Fan and
                  Meihao Fan and
                  Guoliang Li and
                  Xiaoyong Du},
  title        = {DeepAnalyze: Agentic Large Language Models for Autonomous Data Science},
  journal      = {CoRR},
  volume       = {abs/2510.16872},
  year         = {2025}
}

@article{DBLP:journals/corr/SchulmanWDRK17,
  author       = {John Schulman and
                  Filip Wolski and
                  Prafulla Dhariwal and
                  Alec Radford and
                  Oleg Klimov},
  title        = {Proximal Policy Optimization Algorithms},
  journal      = {CoRR},
  volume       = {abs/1707.06347},
  year         = {2017}
}

@article{DBLP:journals/corr/abs-2402-03300,
  author       = {Zhihong Shao and
                  Peiyi Wang and
                  Qihao Zhu and
                  Runxin Xu and
                  Junxiao Song and
                  Mingchuan Zhang and
                  Y. K. Li and
                  Y. Wu and
                  Daya Guo},
  title        = {DeepSeekMath: Pushing the Limits of Mathematical Reasoning in Open Language Models},
  journal      = {CoRR},
  volume       = {abs/2402.03300},
  year         = {2024}
}

@inproceedings{DBLP:conf/acl/AhmadianCGFKPUH24,
  author       = {Arash Ahmadian and
                  Chris Cremer and
                  Matthias Gall{\'{e}} and
                  Marzieh Fadaee and
                  Julia Kreutzer and
                  Olivier Pietquin and
                  Ahmet {\"{U}}st{\"{u}}n and
                  Sara Hooker},
  title        = {Back to Basics: Revisiting REINFORCE-Style Optimization for Learning from Human Feedback in LLMs},
  booktitle    = {{ACL}},
  pages        = {12248--12267},
  year         = {2024}
}

@article{DBLP:journals/corr/abs-2501-03262,
  author       = {Jian Hu},
  title        = {{REINFORCE++:} {A} Simple and Efficient Approach for Aligning Large Language Models},
  journal      = {CoRR},
  volume       = {abs/2501.03262},
  year         = {2025}
}

@article{DBLP:journals/corr/abs-2511-13288,
  author       = {Haoyang Hong and
                  Jiajun Yin and
                  Yuan Wang and
                  Jingnan Liu and
                  Zhe Chen and
                  Ailing Yu and
                  Ji Li and
                  Zhiling Ye and
                  Hansong Xiao and
                  Yefei Chen and
                  Hualei Zhou and
                  Yun Yue and
                  Minghui Yang and
                  Chunxiao Guo and
                  Junwei Liu and
                  Peng Wei and
                  Jinjie Gu},
  title        = {Multi-Agent Deep Research: Training Multi-Agent Systems with {M-GRPO}},
  journal      = {CoRR},
  volume       = {abs/2511.13288},
  year         = {2025}
}

@article{DBLP:journals/corr/abs-2510-04678,
  author       = {Zhanfeng Mo and
                  Xingxuan Li and
                  Yuntao Chen and
                  Lidong Bing},
  title        = {Multi-Agent Tool-Integrated Policy Optimization},
  journal      = {CoRR},
  volume       = {abs/2510.04678},
  year         = {2025}
}

@inproceedings{DBLP:conf/icml/HuZWCM0WSXZCY0K24,
  author       = {Xueyu Hu and
                  Ziyu Zhao and
                  Shuang Wei and
                  Ziwei Chai and
                  Qianli Ma and
                  Guoyin Wang and
                  Xuwu Wang and
                  Jing Su and
                  Jingjing Xu and
                  Ming Zhu and
                  Yao Cheng and
                  Jianbo Yuan and
                  Jiwei Li and
                  Kun Kuang and
                  Yang Yang and
                  Hongxia Yang and
                  Fei Wu},
  title        = {InfiAgent-DABench: Evaluating Agents on Data Analysis Tasks},
  booktitle    = {{ICML}},
  pages        = {19544--19572},
  year         = {2024}
}

@inproceedings{DBLP:conf/emnlp/HuangLYZLWHHLZL24,
  author       = {Yiming Huang and
                  Jianwen Luo and
                  Yan Yu and
                  Yitong Zhang and
                  Fangyu Lei and
                  Yifan Wei and
                  Shizhu He and
                  Lifu Huang and
                  Xiao Liu and
                  Jun Zhao and
                  Kang Liu},
  title        = {DA-Code: Agent Data Science Code Generation Benchmark for Large Language Models},
  booktitle    = {{EMNLP}},
  pages        = {13487--13521},
  year         = {2024}
}

@inproceedings{DBLP:conf/iclr/ChenCNZWYLLWLDX25,
  author       = {Ziru Chen and
                  Shijie Chen and
                  Yuting Ning and
                  Qianheng Zhang and
                  Boshi Wang and
                  Botao Yu and
                  Yifei Li and
                  Zeyi Liao and
                  Chen Wei and
                  Zitong Lu and
                  Vishal Dey and
                  Mingyi Xue and
                  Frazier N. Baker and
                  Benjamin Burns and
                  Daniel Adu{-}Ampratwum and
                  Xuhui Huang and
                  Xia Ning and
                  Song Gao and
                  Yu Su and
                  Huan Sun},
  title        = {ScienceAgentBench: Toward Rigorous Assessment of Language Agents for Data-Driven Scientific Discovery},
  booktitle    = {{ICLR}},
  year         = {2025}
}

@inproceedings{
qiang2025mledojo,
title={{MLE}-Dojo: Interactive Environments for Empowering {LLM} Agents in Machine Learning Engineering},
author={Rushi Qiang and Yuchen Zhuang and Yinghao Li and Dingu Sagar V K and Rongzhi Zhang and ChangHao Li and Ian Shu-Hei Wong and Sherry Yang and Percy Liang and Chao Zhang and Bo Dai},
booktitle={{NeurIPS}},
year={2025}
}

@inproceedings{DBLP:conf/acl/YangZWCHYLTLYLS24,
  author       = {Zhiyu Yang and
                  Zihan Zhou and
                  Shuo Wang and
                  Xin Cong and
                  Xu Han and
                  Yukun Yan and
                  Zhenghao Liu and
                  Zhixing Tan and
                  Pengyuan Liu and
                  Dong Yu and
                  Zhiyuan Liu and
                  Xiaodong Shi and
                  Maosong Sun},
  title        = {MatPlotAgent: Method and Evaluation for LLM-Based Agentic Scientific Data Visualization},
  booktitle    = {{ACL} (Findings)},
  pages        = {11789--11804},
  year         = {2024}
}

@misc{openai2023,
  author       = {OpenAI},
  title        = {Code Interpreter},
  year         = {2023},
  url          = {https://platform.openai.com/docs/guides/tools-code-interpreter}
}

@article{Sun17072025,
author = {Maojun Sun and Ruijian Han and Binyan Jiang and Houduo Qi and Defeng Sun and Yancheng Yuan and Jian Huang},
title = {LAMBDA: A Large Model Based Data Agent},
journal = {J. Am. Stat. Assoc.},
volume = {0},
number = {0},
pages = {1--13},
year = {2025},
}

@inproceedings{DBLP:journals/corr/abs-2509-25084,
  author       = {Shuofei Qiao and
                  Yanqiu Zhao and
                  Zhisong Qiu and
                  Xiaobin Wang and
                  Jintian Zhang and
                  Zhao Bin and
                  Ningyu Zhang and
                  Yong Jiang and
                  Pengjun Xie and
                  Fei Huang and
                  Huajun Chen},
  title        = {Scaling Generalist Data-Analytic Agents},
  booktitle={{ICLR}},
  year         = {2026}
}

@misc{deepseek2025v31,
  author       = {DeepSeek},
  title        = {DeepSeek-V3.1 Release},
  year         = {2025},
  url          = {https://api-docs.deepseek.com/news/news250821}
}

@misc{openai2023gpt4,
  author       = {OpenAI},
  title        = {Hello GPT-4},
  year         = {2023},
  url          = {https://openai.com/zh-Hans-CN/index/hello-gpt-4o/}
}

@misc{openai2025o3o4mini,
  author       = {OpenAI},
  title        = {Introducing OpenAI o3 and o4-mini},
  year         = {2025},
  url          = {https://openai.com/zh-Hans-CN/index/introducing-o3-and-o4-mini/}
}

@article{DBLP:journals/corr/abs-2505-09388,
  author       = {An Yang and
                  Anfeng Li and
                  Baosong Yang and
                  Beichen Zhang and
                  Binyuan Hui and
                  Bo Zheng and
                  Bowen Yu and
                  Chang Gao and
                  Chengen Huang and
                  Chenxu Lv and
                  Chujie Zheng and
                  Dayiheng Liu and
                  Fan Zhou and
                  Fei Huang and
                  Feng Hu and
                  Hao Ge and
                  Haoran Wei and
                  Huan Lin and
                  Jialong Tang and
                  Jian Yang and
                  Jianhong Tu and
                  Jianwei Zhang and
                  Jian Yang and
                  Jiaxi Yang and
                  Jingren Zhou and
                  Junyang Lin and
                  Kai Dang and
                  Keqin Bao and
                  Kexin Yang and
                  Le Yu and
                  Lianghao Deng and
                  Mei Li and
                  Mingfeng Xue and
                  Mingze Li and
                  Pei Zhang and
                  Peng Wang and
                  Qin Zhu and
                  Rui Men and
                  Ruize Gao and
                  Shixuan Liu and
                  Shuang Luo and
                  Tianhao Li and
                  Tianyi Tang and
                  Wenbiao Yin and
                  Xingzhang Ren and
                  Xinyu Wang and
                  Xinyu Zhang and
                  Xuancheng Ren and
                  Yang Fan and
                  Yang Su and
                  Yichang Zhang and
                  Yinger Zhang and
                  Yu Wan and
                  Yuqiong Liu and
                  Zekun Wang and
                  Zeyu Cui and
                  Zhenru Zhang and
                  Zhipeng Zhou and
                  Zihan Qiu},
  title        = {Qwen3 Technical Report},
  journal      = {CoRR},
  volume       = {abs/2505.09388},
  year         = {2025}
}

@inproceedings{DBLP:conf/iclr/JingHWYYM0DY25,
  author       = {Liqiang Jing and
                  Zhehui Huang and
                  Xiaoyang Wang and
                  Wenlin Yao and
                  Wenhao Yu and
                  Kaixin Ma and
                  Hongming Zhang and
                  Xinya Du and
                  Dong Yu},
  title        = {DSBench: How Far Are Data Science Agents from Becoming Data Science Experts?},
  booktitle    = {{ICLR}},
  year         = {2025}
}

@article{DBLP:journals/corr/abs-2509-01055,
  author       = {Dongfu Jiang and
                  Yi Lu and
                  Zhuofeng Li and
                  Zhiheng Lyu and
                  Ping Nie and
                  Haozhe Wang and
                  Alex Su and
                  Hui Chen and
                  Kai Zou and
                  Chao Du and
                  Tianyu Pang and
                  Wenhu Chen},
  title        = {VerlTool: Towards Holistic Agentic Reinforcement Learning with Tool Use},
  journal      = {CoRR},
  volume       = {abs/2509.01055},
  year         = {2025}
}

@article{sun2025survey,
  title={A survey on large language model-based agents for statistics and data science},
  author={Sun, Maojun and Han, Ruijian and Jiang, Binyan and Qi, Houduo and Sun, Defeng and Yuan, Yancheng and Huang, Jian},
  journal={Am. Stat.},
  volume = {0},
  number = {0},
  pages={1--14},
  year={2025}
}

@article{DBLP:journals/corr/abs-2508-02744,
  author       = {Peiran Wang and
                  Yaoning Yu and
                  Ke Chen and
                  Xianyang Zhan and
                  Haohan Wang},
  title        = {Large Language Model-based Data Science Agent: {A} Survey},
  journal      = {CoRR},
  volume       = {abs/2508.02744},
  year         = {2025}
}

@article{DBLP:journals/corr/abs-2509-23988,
  author       = {Zirui Tang and
                  Weizheng Wang and
                  Zihang Zhou and
                  Yang Jiao and
                  Bangrui Xu and
                  Boyu Niu and
                  Xuanhe Zhou and
                  Guoliang Li and
                  Yeye He and
                  Wei Zhou and
                  Yitong Song and
                  Cheng Tan and
                  Bin Wang and
                  Conghui He and
                  Xiaoyang Wang and
                  Fan Wu},
  title        = {LLM/Agent-as-Data-Analyst: {A} Survey},
  journal      = {CoRR},
  volume       = {abs/2509.23988},
  year         = {2025}
}

@article{DBLP:journals/corr/abs-2510-23587,
  author       = {Yizhang Zhu and
                  Liangwei Wang and
                  Chenyu Yang and
                  Xiaotian Lin and
                  Boyan Li and
                  Wei Zhou and
                  Xinyu Liu and
                  Zhangyang Peng and
                  Tianqi Luo and
                  Yu Li and
                  Chengliang Chai and
                  Chong Chen and
                  Shimin Di and
                  Ju Fan and
                  Ji Sun and
                  Nan Tang and
                  Fugee Tsung and
                  Jiannan Wang and
                  Chenglin Wu and
                  Yanwei Xu and
                  Shaolei Zhang and
                  Yong Zhang and
                  Xuanhe Zhou and
                  Guoliang Li and
                  Yuyu Luo},
  title        = {A Survey of Data Agents: Emerging Paradigm or Overstated Hype?},
  journal      = {CoRR},
  volume       = {abs/2510.23587},
  year         = {2025}
}

@article{DBLP:journals/cacm/BieRHHSW22,
  author       = {Tijl De Bie and
                  Luc De Raedt and
                  Jos{\'{e}} Hern{\'{a}}ndez{-}Orallo and
                  Holger H. Hoos and
                  Padhraic Smyth and
                  Christopher K. I. Williams},
  title        = {Automating data science},
  journal      = {Commun. {ACM}},
  volume       = {65},
  number       = {3},
  pages        = {76--87},
  year         = {2022}
}

@article{MUMUNI2025113,
title = {Automated data processing and feature engineering for deep learning and big data applications: A survey},
journal = {J. Inf. Intell.},
volume = {3},
number = {2},
pages = {113-153},
year = {2025},
author = {Alhassan Mumuni and Fuseini Mumuni}
}

@article{DBLP:journals/jair/ZollerH21,
  author       = {Marc{-}Andr{\'{e}} Z{\"{o}}ller and
                  Marco F. Huber},
  title        = {Benchmark and Survey of Automated Machine Learning Frameworks},
  journal      = {J. Artif. Intell. Res.},
  volume       = {70},
  pages        = {409--472},
  year         = {2021}
}

@article{DBLP:journals/fcsc/WangMFZYZCTCLZWW24,
  author       = {Lei Wang and
                  Chen Ma and
                  Xueyang Feng and
                  Zeyu Zhang and
                  Hao Yang and
                  Jingsen Zhang and
                  Zhiyuan Chen and
                  Jiakai Tang and
                  Xu Chen and
                  Yankai Lin and
                  Wayne Xin Zhao and
                  Zhewei Wei and
                  Jirong Wen},
  title        = {A survey on large language model based autonomous agents},
  journal      = {Frontiers Comput. Sci.},
  volume       = {18},
  number       = {6},
  pages        = {186345},
  year         = {2024}
}

@inproceedings{DBLP:conf/ijcai/GuoCWCPCW024,
  author       = {Taicheng Guo and
                  Xiuying Chen and
                  Yaqi Wang and
                  Ruidi Chang and
                  Shichao Pei and
                  Nitesh V. Chawla and
                  Olaf Wiest and
                  Xiangliang Zhang},
  title        = {Large Language Model Based Multi-agents: {A} Survey of Progress and Challenges},
  booktitle    = {{IJCAI}},
  pages        = {8048--8057},
  year         = {2024}
}

@article{DBLP:journals/corr/abs-2510-04023,
  author       = {Mizanur Rahman and
                  Amran Bhuiyan and
                  Mohammed Saidul Islam and
                  Md. Tahmid Rahman Laskar and
                  Ridwan Mahbub and
                  Ahmed Masry and
                  Shafiq Joty and
                  Enamul Hoque},
  title        = {LLM-Based Data Science Agents: {A} Survey of Capabilities, Challenges, and Future Directions},
  journal      = {CoRR},
  volume       = {abs/2510.04023},
  year         = {2025}
}

@inproceedings{
zhang2024datacopilot,
title={Data-Copilot: Bridging Billions of Data and Humans with Autonomous Workflow},
author={Wenqi Zhang and Yongliang Shen and Weiming Lu and Yueting Zhuang},
booktitle={LLMAgents@ICLR},
year={2024}
}

@inproceedings{DBLP:conf/www/ChenYZ000HMZ25,
  author       = {Yibin Chen and
                  Yifu Yuan and
                  Zeyu Zhang and
                  Yan Zheng and
                  Jinyi Liu and
                  Fei Ni and
                  Jianye Hao and
                  Hangyu Mao and
                  Fuzheng Zhang},
  title        = {SheetAgent: Towards a Generalist Agent for Spreadsheet Reasoning and Manipulation via Large Language Models},
  booktitle    = {{WWW}},
  pages        = {158--177},
  year         = {2025}
}

@article{DBLP:journals/corr/abs-2601-10402,
  author       = {Xinyu Zhu and
                  Yuzhu Cai and
                  Zexi Liu and
                  Bingyang Zheng and
                  Cheng Wang and
                  Rui Ye and
                  Jiaao Chen and
                  Hanrui Wang and
                  Wei{-}Chen Wang and
                  Yuzhi Zhang and
                  Linfeng Zhang and
                  Weinan E and
                  Di Jin and
                  Siheng Chen and
                  Yanfeng Wang},
  title        = {Toward Ultra-Long-Horizon Agentic Science: Cognitive Accumulation for Machine Learning Engineering},
  journal      = {CoRR},
  volume       = {abs/2601.10402},
  year         = {2026}
}

@article{DBLP:journals/corr/abs-2506-16499,
  author       = {Zexi Liu and
                  Yuzhu Cai and
                  Xinyu Zhu and
                  Yujie Zheng and
                  Runkun Chen and
                  Ying Wen and
                  Yanfeng Wang and
                  Weinan E and
                  Siheng Chen},
  title        = {ML-Master: Towards AI-for-AI via Integration of Exploration and Reasoning},
  journal      = {CoRR},
  volume       = {abs/2506.16499},
  year         = {2025}
}

@article{DBLP:journals/corr/abs-2510-08511,
  author       = {Shangheng Du and
                  Xiangchao Yan and
                  Dengyang Jiang and
                  Jiakang Yuan and
                  Yusong Hu and
                  Xin Li and
                  Liang He and
                  Bo Zhang and
                  Lei Bai},
  title        = {AutoMLGen: Navigating Fine-Grained Optimization for Coding Agents},
  journal      = {CoRR},
  volume       = {abs/2510.08511},
  year         = {2025}
}

@inproceedings{
nam2025mlestar,
title={{MLE}-{STAR}: Machine Learning Engineering Agent via Search and Targeted Refinement},
author={Jaehyun Nam and Jinsung Yoon and Jiefeng Chen and Jinwoo Shin and Sercan O Arik and Tomas Pfister},
booktitle={NeurIPS},
year={2025}
}

@inproceedings{
fang2025mlzero,
title={{MLZ}ero: A Multi-Agent System for End-to-end Machine Learning Automation},
author={Haoyang Fang and Boran Han and Nick Erickson and Xiyuan Zhang and Su Zhou and Anirudh Dagar and Jiani Zhang and Ali Caner Turkmen and Cuixiong Hu and Huzefa Rangwala and Ying Nian Wu and Bernie Wang and George Karypis},
booktitle={NeurIPS},
year={2025}
}

@article{DBLP:journals/corr/abs-2505-23723,
  author       = {Zexi Liu and
                  Jingyi Chai and
                  Xinyu Zhu and
                  Shuo Tang and
                  Rui Ye and
                  Bolun Zhang and
                  Lei Bai and
                  Siheng Chen},
  title        = {ML-Agent: Reinforcing {LLM} Agents for Autonomous Machine Learning Engineering},
  journal      = {CoRR},
  volume       = {abs/2505.23723},
  year         = {2025}
}

@inproceedings{DBLP:conf/nips/Bo0DFW00W24,
  author       = {Xiaohe Bo and
                  Zeyu Zhang and
                  Quanyu Dai and
                  Xueyang Feng and
                  Lei Wang and
                  Rui Li and
                  Xu Chen and
                  Ji{-}Rong Wen},
  title        = {Reflective Multi-Agent Collaboration based on Large Language Models},
  pages        = {138595-138631},
  booktitle    = {NeurIPS},
  year         = {2024}
}

@inproceedings{
motwani2025malt,
title={{MALT}: Improving Reasoning with Multi-Agent {LLM} Training},
author={Sumeet Ramesh Motwani and Chandler Smith and Rocktim Jyoti Das and Rafael Rafailov and Philip Torr and Ivan Laptev and Fabio Pizzati and Ronald Clark and Christian Schroeder de Witt},
booktitle={COLM},
year={2025}
}

@inproceedings{DBLP:conf/acl/Park0GOZK25,
  author       = {Chanwoo Park and
                  Seungju Han and
                  Xingzhi Guo and
                  Asuman E. Ozdaglar and
                  Kaiqing Zhang and
                  Joo{-}Kyung Kim},
  title        = {MAPoRL: Multi-Agent Post-Co-Training for Collaborative Large Language Models with Reinforcement Learning},
  booktitle    = {{ACL}},
  pages        = {30215--30248},
  year         = {2025}
}

@inproceedings{DBLP:conf/acl/HongLLLWZLCZWZZ25,
  author       = {Sirui Hong and
                  Yizhang Lin and
                  Bang Liu and
                  Bangbang Liu and
                  Binhao Wu and
                  Ceyao Zhang and
                  Danyang Li and
                  Jiaqi Chen and
                  Jiayi Zhang and
                  Jinlin Wang and
                  Li Zhang and
                  Lingyao Zhang and
                  Min Yang and
                  Mingchen Zhuge and
                  Taicheng Guo and
                  Tuo Zhou and
                  Wei Tao and
                  Robert Tang and
                  Xiangtao Lu and
                  Xiawu Zheng and
                  Xinbing Liang and
                  Yaying Fei and
                  Yuheng Cheng and
                  Yongxin Ni and
                  Zhibin Gou and
                  Zongze Xu and
                  Yuyu Luo and
                  Chenglin Wu},
  title        = {Data Interpreter: An {LLM} Agent for Data Science},
  booktitle    = {{ACL} (Findings)},
  pages        = {19796--19821},
  year         = {2025}
}

@article{DBLP:journals/corr/abs-2512-00672,
  author       = {Yaswanth Chittepu and
                  Raghavendra Addanki and
                  Tung Mai and
                  Anup B. Rao and
                  Branislav Kveton},
  title        = {ML-Tool-Bench: Tool-Augmented Planning for {ML} Tasks},
  journal      = {CoRR},
  volume       = {abs/2512.00672},
  year         = {2025}
}

@article{DBLP:journals/corr/abs-2509-06283,
  author       = {Xuan{-}Phi Nguyen and
                  Shrey Pandit and
                  Revanth Gangi Reddy and
                  Austin Xu and
                  Silvio Savarese and
                  Caiming Xiong and
                  Shafiq Joty},
  title        = {SFR-DeepResearch: Towards Effective Reinforcement Learning for Autonomously Reasoning Single Agents},
  journal      = {CoRR},
  volume       = {abs/2509.06283},
  year         = {2025}
}

@article{DBLP:journals/corr/abs-2510-00615,
  author       = {Minki Kang and
                  Wei{-}Ning Chen and
                  Dongge Han and
                  Huseyin A. Inan and
                  Lukas Wutschitz and
                  Yanzhi Chen and
                  Robert Sim and
                  Saravan Rajmohan},
  title        = {{ACON:} Optimizing Context Compression for Long-horizon {LLM} Agents},
  journal      = {CoRR},
  volume       = {abs/2510.00615},
  year         = {2025}
}

@article{DBLP:journals/corr/abs-2507-21046,
  author       = {Huan{-}ang Gao and
                  Jiayi Geng and
                  Wenyue Hua and
                  Mengkang Hu and
                  Xinzhe Juan and
                  Hongzhang Liu and
                  Shilong Liu and
                  Jiahao Qiu and
                  Xuan Qi and
                  Yiran Wu and
                  Hongru Wang and
                  Han Xiao and
                  Yuhang Zhou and
                  Shaokun Zhang and
                  Jiayi Zhang and
                  Jinyu Xiang and
                  Yixiong Fang and
                  Qiwen Zhao and
                  Dongrui Liu and
                  Qihan Ren and
                  Cheng Qian and
                  Zhenhailong Wang and
                  Minda Hu and
                  Huazheng Wang and
                  Qingyun Wu and
                  Heng Ji and
                  Mengdi Wang},
  title        = {A Survey of Self-Evolving Agents: On Path to Artificial Super Intelligence},
  journal      = {CoRR},
  volume       = {abs/2507.21046},
  year         = {2025}
}

@article{DBLP:journals/corr/abs-2508-07407,
  author       = {Jinyuan Fang and
                  Yanwen Peng and
                  Xi Zhang and
                  Yingxu Wang and
                  Xinhao Yi and
                  Guibin Zhang and
                  Yi Xu and
                  Bin Wu and
                  Siwei Liu and
                  Zihao Li and
                  Zhaochun Ren and
                  Nikos Aletras and
                  Xi Wang and
                  Han Zhou and
                  Zaiqiao Meng},
  title        = {A Comprehensive Survey of Self-Evolving {AI} Agents: {A} New Paradigm Bridging Foundation Models and Lifelong Agentic Systems},
  journal      = {CoRR},
  volume       = {abs/2508.07407},
  year         = {2025}
}

@inproceedings{
xu2025amem,
title={A-Mem: Agentic Memory for {LLM} Agents},
author={Wujiang Xu and Zujie Liang and Kai Mei and Hang Gao and Juntao Tan and Yongfeng Zhang},
booktitle={NeurIPS},
year={2025}
}

@inproceedings{DBLP:conf/aaai/ZhongGGYW24,
  author       = {Wanjun Zhong and
                  Lianghong Guo and
                  Qiqi Gao and
                  He Ye and
                  Yanlin Wang},
  title        = {MemoryBank: Enhancing Large Language Models with Long-Term Memory},
  booktitle    = {{AAAI}},
  pages        = {19724--19731},
  year         = {2024}
}

@article{DBLP:journals/corr/abs-2504-19413,
  author       = {Prateek Chhikara and
                  Dev Khant and
                  Saket Aryan and
                  Taranjeet Singh and
                  Deshraj Yadav},
  title        = {Mem0: Building Production-Ready {AI} Agents with Scalable Long-Term Memory},
  journal      = {CoRR},
  volume       = {abs/2504.19413},
  year         = {2025}
}

@inproceedings{DBLP:conf/icml/WangMFN25,
  author       = {Zora Zhiruo Wang and
                  Jiayuan Mao and
                  Daniel Fried and
                  Graham Neubig},
  title        = {Agent Workflow Memory},
  booktitle    = {{ICML}},
  pages        = {63897-63911},
  year         = {2025}
}

@article{DBLP:journals/corr/abs-2406-07496,
  author       = {Mert Y{\"{u}}ksekg{\"{o}}n{\"{u}}l and
                  Federico Bianchi and
                  Joseph Boen and
                  Sheng Liu and
                  Zhi Huang and
                  Carlos Guestrin and
                  James Zou},
  title        = {TextGrad: Automatic "Differentiation" via Text},
  journal      = {CoRR},
  volume       = {abs/2406.07496},
  year         = {2024}
}

@inproceedings{DBLP:conf/iclr/Guo0GLS0L0Y24,
  author       = {Qingyan Guo and
                  Rui Wang and
                  Junliang Guo and
                  Bei Li and
                  Kaitao Song and
                  Xu Tan and
                  Guoqing Liu and
                  Jiang Bian and
                  Yujiu Yang},
  title        = {Connecting Large Language Models with Evolutionary Algorithms Yields Powerful Prompt Optimizers},
  booktitle    = {{ICLR}},
  year         = {2024}
}

@inproceedings{DBLP:conf/icml/FernandoBMOR24,
  author       = {Chrisantha Fernando and
                  Dylan Banarse and
                  Henryk Michalewski and
                  Simon Osindero and
                  Tim Rockt{\"{a}}schel},
  title        = {Promptbreeder: Self-Referential Self-Improvement via Prompt Evolution},
  booktitle    = {{ICML}},
  pages        = {13481-13544},
  year         = {2024}
}

@inproceedings{DBLP:conf/iclr/YuanC000J24,
  author       = {Lifan Yuan and
                  Yangyi Chen and
                  Xingyao Wang and
                  Yi Fung and
                  Hao Peng and
                  Heng Ji},
  title        = {{CRAFT:} Customizing LLMs by Creating and Retrieving from Specialized Toolsets},
  booktitle    = {{ICLR}},
  year         = {2024}
}

@inproceedings{
lu2026dont,
title={Don't Just Fine-tune the Agent, Tune the Environment},
author={Siyuan Lu and Zechuan Wang and Hongxuan Zhang and Qintong Wu and Leilei Gan and Chenyi Zhuang and Jinjie Gu and Tao Lin},
booktitle={{ICLR}},
year={2026}
}

@inproceedings{DBLP:conf/icml/ZhugeWKFKS24,
  author       = {Mingchen Zhuge and
                  Wenyi Wang and
                  Louis Kirsch and
                  Francesco Faccio and
                  Dmitrii Khizbullin and
                  J{\"{u}}rgen Schmidhuber},
  title        = {GPTSwarm: Language Agents as Optimizable Graphs},
  booktitle    = {{ICML}},
  pages        = {62743-62767},
  year         = {2024}
}

@inproceedings{DBLP:conf/iclr/ZhangXYTCCZCHWZ25,
  author       = {Jiayi Zhang and
                  Jinyu Xiang and
                  Zhaoyang Yu and
                  Fengwei Teng and
                  Xionghui Chen and
                  Jiaqi Chen and
                  Mingchen Zhuge and
                  Xin Cheng and
                  Sirui Hong and
                  Jinlin Wang and
                  Bingnan Zheng and
                  Bang Liu and
                  Yuyu Luo and
                  Chenglin Wu},
  title        = {AFlow: Automating Agentic Workflow Generation},
  booktitle    = {{ICLR}},
  year         = {2025}
}

@inproceedings{DBLP:conf/icml/ZhangNF00025,
  author       = {Guibin Zhang and
                  Luyang Niu and
                  Junfeng Fang and
                  Kun Wang and
                  Lei Bai and
                  Xiang Wang},
  title        = {Multi-agent Architecture Search via Agentic Supernet},
  booktitle    = {{ICML}},
  pages        = {75834-75852},
  year         = {2025}
}

@inproceedings{DBLP:conf/iclr/ShangLZMLXL25,
  author       = {Yu Shang and
                  Yu Li and
                  Keyu Zhao and
                  Likai Ma and
                  Jiahe Liu and
                  Fengli Xu and
                  Yong Li},
  title        = {AgentSquare: Automatic {LLM} Agent Search in Modular Design Space},
  booktitle    = {{ICLR}},
  year         = {2025}
}

@article{DBLP:journals/corr/abs-2507-13334,
  author       = {Lingrui Mei and
                  Jiayu Yao and
                  Yuyao Ge and
                  Yiwei Wang and
                  Baolong Bi and
                  Yujun Cai and
                  Jiazhi Liu and
                  Mingyu Li and
                  Zhong{-}Zhi Li and
                  Duzhen Zhang and
                  Chenlin Zhou and
                  Jiayi Mao and
                  Tianze Xia and
                  Jiafeng Guo and
                  Shenghua Liu},
  title        = {A Survey of Context Engineering for Large Language Models},
  journal      = {CoRR},
  volume       = {abs/2507.13334},
  year         = {2025}
}

@inproceedings{DBLP:conf/acl/ChenLPSSZGY25,
  author       = {Mingda Chen and
                  Yang Li and
                  Karthik Padthe and
                  Rulin Shao and
                  Alicia Yi Sun and
                  Luke Zettlemoyer and
                  Gargi Ghosh and
                  Wen{-}tau Yih},
  title        = {Improving Factuality with Explicit Working Memory},
  booktitle    = {{ACL}},
  pages        = {11199--11213},
  year         = {2025}
}

@inproceedings{
zhang2025gmemory,
title={G-Memory: Tracing Hierarchical Memory for Multi-Agent Systems},
author={Guibin Zhang and Muxin Fu and Kun Wang and Guancheng Wan and Miao Yu and Shuicheng YAN},
booktitle={NeurIPS},
year={2025}
}

@inproceedings{DBLP:conf/icml/LeeCFCF24,
  author       = {Kuang{-}Huei Lee and
                  Xinyun Chen and
                  Hiroki Furuta and
                  John F. Canny and
                  Ian Fischer},
  title        = {A Human-Inspired Reading Agent with Gist Memory of Very Long Contexts},
  booktitle    = {{ICML}},
  pages        = {26396-26415},
  year         = {2024}
}

@inproceedings{DBLP:conf/acl/FeiNZH0D024,
  author       = {Weizhi Fei and
                  Xueyan Niu and
                  Pingyi Zhou and
                  Lu Hou and
                  Bo Bai and
                  Lei Deng and
                  Wei Han},
  title        = {Extending Context Window of Large Language Models via Semantic Compression},
  booktitle    = {{ACL} (Findings)},
  pages        = {5169--5181},
  year         = {2024}
}

@inproceedings{DBLP:conf/iclr/00010WWCW24,
  author       = {Tao Ge and
                  Jing Hu and
                  Lei Wang and
                  Xun Wang and
                  Si{-}Qing Chen and
                  Furu Wei},
  title        = {In-context Autoencoder for Context Compression in a Large Language Model},
  booktitle    = {{ICLR}},
  year         = {2024}
}

@inproceedings{
he2024camelot,
title={{CAMEL}oT: Towards Large Language Models with Training-Free Consolidated Associative Memory},
author={Zexue He and Leonid Karlinsky and Donghyun Kim and Julian McAuley and Dmitry Krotov and Rogerio Feris},
booktitle={LCFM@ICML},
year={2024}
}

@article{DBLP:journals/corr/abs-2509-02547,
  author       = {Guibin Zhang and
                  Hejia Geng and
                  Xiaohang Yu and
                  Zhenfei Yin and
                  Zaibin Zhang and
                  Zelin Tan and
                  Heng Zhou and
                  Zhongzhi Li and
                  Xiangyuan Xue and
                  Yijiang Li and
                  Yifan Zhou and
                  Yang Chen and
                  Chen Zhang and
                  Yutao Fan and
                  Zihu Wang and
                  Songtao Huang and
                  Yue Liao and
                  Hongru Wang and
                  Mengyue Yang and
                  Heng Ji and
                  Michael Littman and
                  Jun Wang and
                  Shuicheng Yan and
                  Philip Torr and
                  Lei Bai},
  title        = {The Landscape of Agentic Reinforcement Learning for LLMs: {A} Survey},
  journal      = {CoRR},
  volume       = {abs/2509.02547},
  year         = {2025}
}

@article{he2021automl,
  title={AutoML: A survey of the state-of-the-art},
  author={He, Xin and Zhao, Kaiyong and Chu, Xiaowen},
  journal={Knowl-based Syst},
  volume={212},
  pages={106622},
  year={2021}
}

@article{DBLP:journals/corr/abs-2510-11967,
  author       = {Weiwei Sun and
                  Miao Lu and
                  Zhan Ling and
                  Kang Liu and
                  Xuesong Yao and
                  Yiming Yang and
                  Jiecao Chen},
  title        = {Scaling Long-Horizon {LLM} Agent via Context-Folding},
  journal      = {CoRR},
  volume       = {abs/2510.11967},
  year         = {2025}
}

@article{DBLP:journals/corr/abs-2512-22733,
  author       = {Jiaqi Shao and
                  Yufeng Miao and
                  Wei Zhang and
                  Bing Luo},
  title        = {FoldAct: Efficient and Stable Context Folding for Long-Horizon Search Agents},
  journal      = {CoRR},
  volume       = {abs/2512.22733},
  year         = {2025}
}

@article{DBLP:journals/corr/abs-2603-01145,
  author       = {Yutao Yang and
                  Junsong Li and
                  Qianjun Pan and
                  Bihao Zhan and
                  Yuxuan Cai and
                  Lin Du and
                  Jie Zhou and
                  Kai Chen and
                  Qin Chen and
                  Xin Li and
                  Bo Zhang and
                  Liang He},
  title        = {AutoSkill: Experience-Driven Lifelong Learning via Skill Self-Evolution},
  journal      = {CoRR},
  volume       = {abs/2603.01145},
  year         = {2026}
}

@inproceedings{DBLP:conf/kdd/0002CWW00Z0DLP025,
  author       = {Chuan Qin and
                  Xin Chen and
                  Chengrui Wang and
                  Pengmin Wu and
                  Xi Chen and
                  Yihang Cheng and
                  Jingyi Zhao and
                  Meng Xiao and
                  Xiangchao Dong and
                  Qingqing Long and
                  Boya Pan and
                  Han Wu and
                  Chengzan Li and
                  Yuanchun Zhou and
                  Hui Xiong and
                  Hengshu Zhu},
  title        = {SciHorizon: Benchmarking AI-for-Science Readiness from Scientific Data to Large Language Models},
  booktitle    = {{KDD} {(2)}},
  pages        = {5754--5765},
  year         = {2025}
}

@inproceedings{DBLP:conf/iclr/Cai00CZ24,
  author       = {Tianle Cai and
                  Xuezhi Wang and
                  Tengyu Ma and
                  Xinyun Chen and
                  Denny Zhou},
  title        = {Large Language Models as Tool Makers},
  booktitle    = {{ICLR}},
  year         = {2024}
}


\appendix
\section{Theoretical Analysis}
\subsection{Notations}
Before presenting the theoretical analysis, we introduce the notations used throughout this section.
Let $C \in \mathcal{C}$ denote the task context accumulated by the agent, and let $\mathcal{A}$ denote an LLM-based agent.
Let $\mathcal{T} = \{t_1,\dots,t_K\}$ be the global tool set available to agent $\mathcal{A}$. We define the \emph{tool selection} problem as follows:

\begin{definition}
Given a context $C$, the tool selection problem is to select the most appropriate tool from the tool set $\mathcal{T}$ to solve the task.
Formally,
$
\hat{t} = \pi_\theta(C, \mathcal{T}),
$
where $\pi_\theta$ denotes the agent policy parameterized by $\theta$.
The agent selects tools based on an internal scoring function conditioned on a context representation $\phi(C)$.
\end{definition}

We define a \emph{base agent} as an agent that selects tools directly from the global tool set $\mathcal{T}$. For our hierarchical agent, we assume that the tool set is partitioned into $N$ disjoint subsets as
$
\mathcal{T} = \bigcup_{j=1}^N \mathcal{T}_j,
\mathcal{T}_i \cap \mathcal{T}_j = \emptyset,
$
where $|\mathcal{T}_j| = k_j$.
A \emph{Manager Agent} first selects a sub-agent indexed by $j(C)$, after which the corresponding sub-agent selects a tool from $\mathcal{T}_{j(C)}$. For a given context $C$, each tool $t$ is associated with a true utility
$
u(t \mid C) \in [0, 1],
$
and the optimal tool is defined as
$
t^*(C) = \arg\max_{t \in \mathcal{T}} u(t \mid C).
$

\subsection{Tool Selection Error Bound for Base Agent}

\begin{assumption}
For the base agent, we assume that its scoring function satisfies $s(t \mid C) = u(t \mid C) + \epsilon_t(C)$, where $\{\epsilon_t(C)\}_{t \in \mathcal{T}}$ are i.i.d. variables with
$
\epsilon_t(C) \sim \mathcal{N}(0, \sigma^2(\phi(C))).
$
The variance $\sigma^2(\phi(C))$ reflects uncertainty induced by the context representation.
The base agent selects a tool according to
$
\hat{t}_{\text{base}}(C) = \arg\max_{t \in \mathcal{T}} s(t \mid C).
$
\end{assumption}

For each context $C$, we define the minimum utility margin as
$
\Delta_{\text{base}}(C)
=
\min_{t \neq t^*(C)}
\left[u(t^*(C)\mid C) - u(t\mid C)\right],
$
where $\Delta_{\text{base}}(C) > 0$ almost surely.

\begin{lemma}
\label{lemma: upbound}
For any fixed context $C$, the base agent’s error probability satisfies
\begin{align}
    \Pr(\hat{t}_{\text{base}}(C) \neq t^*(C) \mid C)
\le (K-1)\exp\left(
-\frac{\Delta_{\text{base}}(C)^2}{4\sigma^2(\phi(C))}
\right).
\end{align}
\end{lemma}

\begin{proof}
For any $t \neq t^*(C)$,
$
\Pr(s(t\mid C) > s(t^*(C)\mid C))
= \Pr(\epsilon_t(C) - \epsilon_{t^*}(C) > \Delta_{\text{base}}(C)).
$
Since $\epsilon_t(C) - \epsilon_{t^*}(C) \sim \mathcal{N}(0, 2\sigma^2(\phi(C)))$, we can get
\begin{align}
    \Pr(\epsilon_t(C) - \epsilon_{t^*}(C) > \Delta_{\text{base}}(C))
\le \exp\left(
-\frac{\Delta_{\text{base}}(C)^2}{4\sigma^2(\phi(C))}
\right).
\end{align}
Applying the union bound over all $K-1$ incorrect candidate tools yields the result.
\end{proof}

\subsection{Tool Selection Error Bound for Hierarchical Agent}
\label{appendix: error_bound}

For our hierarchical agent, the tool selection process can be seen as two tool selection problems. First, the Manager Agent selects a sub-agent. Second, the selected sub-agent chooses a tool from its local tool set. For the Manager Agent, we define the sub-agent utility as
$
U(j \mid C) = \max_{t \in \mathcal{T}_j} u(t \mid C),
$
and the optimal sub-agent as
$
j^*(C) = \arg\max_j U(j \mid C).
$
We define the Manager margin as
$
\Delta_M(C)
=
\min_{j \neq j^*(C)}
\left[ U(j^*(C)\mid C) - U(j \mid C) \right],
$
where $\Delta_M(C) > 0$ almost surely.

\begin{assumption}
\label{assumption: manager_score}
The Manager Agent’s scoring function is
$
S_M(j \mid C) = U(j \mid C) + \eta_j(C),
$
where $\{\eta_j(C)\}$ are i.i.d. random variables with
$
\eta_j(C) \sim \mathcal{N}(0, \sigma_M^2(\phi(C))),
\quad
\sigma_M^2(\phi(C)) \approx \sigma^2(\phi(C)).
$
The Manager Agent selects a sub-agent according to
$
\hat{j}(C) = \arg\max_j S_M(j \mid C).
$
\end{assumption}

\begin{assumption}
Each sub-agent operates on a localized context
$
C_j = h_j(C), \quad \phi_j(C) = \phi(C_j),
$
such that for the optimal sub-agent $j^*(C)$,
$
\sigma^2(\phi_{j^*}(C)) \le \sigma^2(\phi(C)).
$
\end{assumption}
This assumption reflects context specialization. We define the sub-agent margin as
$
\Delta_S(C)
=
\min_{t \in \mathcal{T}_{j^*}\setminus\{t^*\}}
\left[u(t^*(C)\mid C) - u(t\mid C)\right].
$

\begin{assumption}
\label{assumption: sub_score}
The sub-agent’s scoring function satisfies
$
S_S(t \mid C_j) = u(t \mid C) + \epsilon_t(C_j),
$
where $\{\epsilon_t(C_j)\}$ are i.i.d. random variables with
$
\epsilon_t(C_j) \sim \mathcal{N}(0, \sigma_S^2(\phi(C_j))).
$
The sub-agent selects a tool according to
$
\hat{t}_{\text{hier}}(C)
=
\arg\max_{t \in \mathcal{T}_{j^*}}
S_S(t \mid C_j).
$
\end{assumption}

\begin{lemma}
The margin satisfies
\[
\Delta_{\text{base}}(C) = \min\{\Delta_M(C), \Delta_S(C)\}.
\]
\end{lemma}

\begin{proof}
The second-best global tool must either belong to the same subset $\mathcal{T}_{j^*}$, yielding margin $\Delta_S(C)$, or belong to a different subset $\mathcal{T}_j$, yielding margin $\Delta_M(C)$.
Taking the minimum over all $t \neq t^*(C)$ yields the result.
\end{proof}

Based on the Assumption~2, Assumption~4 and Lemma~\ref{lemma: upbound}, we obtain the following corollary.

\begin{corollary}
\label{corollary:pro}
For the tool selection problem, the error probabilities of the Manager Agent and the sub-agent satisfy
\footnotesize{
\begin{align}
    \Pr(\hat{j}(C) \neq j^*(C) \mid C)
&\le
(N-1)\exp\left(
-\frac{\Delta_M(C)^2}{4\sigma_M^2(\phi(C))}
\right),\\
\Pr(\hat{t}_{\text{hier}}(C) \neq t^*(C) \mid \hat{j}=j^*(C), C_j)
&\le
(k_{j^*}-1)\exp\left(
-\frac{\Delta_S(C)^2}{4\sigma_S^2(\phi(C_j))}
\right).
\end{align}}
\end{corollary}

\begin{theorem}
\label{theorem1}
Under Assumptions 1–5, for the tool selection problem under the same context $C$, the hierarchical agent admits a strictly smaller upper bound on the error probability than the base agent.
\end{theorem}

\begin{proof}
The hierarchical agent makes an error if either:
(i) the Manager Agent selects a sub-agent $j \neq j^*(C)$; or
(ii) the Manager Agent selects $j^*(C)$, but the corresponding sub-agent selects a tool $t \neq t^*(C)$. By Corollary~\ref{corollary:pro}, the context-conditioned error probability of the hierarchical agent satisfies

\scriptsize{
\begin{align}
\Pr(\text{error} \mid C)
\le\;&
(N-1)\exp\left(
-\frac{\Delta_M(C)^2}{4\sigma_M^2(\phi(C))}
\right) + (k_{j^*}-1)\exp\left(
-\frac{\Delta_S(C)^2}{4\sigma_S^2(\phi(C_j))}
\right).
\end{align}}\\
\normalsize{
Moreover, since $k_{j^*} < K$, $N < K$,
$
\sigma_S^2(\phi(C_j)) \le \sigma^2(\phi(C))$,
$
\sigma_{M}^2(\phi(C)) \approx \sigma^2(\phi(C))$, and
$
\Delta_{\text{base}}(C) = \min\{\Delta_M(C), \Delta_S(C)\},
$
the above upper bound is strictly smaller than the corresponding error bound of the base agent.}
\end{proof}

\eat{
Theorem~\ref{theorem1} shows that hierarchical tool selection yields a strictly tighter error bound than flat selection under the same context. This improvement stems from decomposing a large decision space into smaller, structured sub-problems, which reduces both action space size and uncertainty in each decision stage. As a result, the hierarchical agent achieves more robust tool selection by increasing effective margins and lowering variance in the scoring process.
}

\subsection{Information-Theoretic Interpretation of the Optimization Objective}
\label{appendix: objective}

Let $Z = g(C)$ denote the compressed context used by the Manager Agent, and let $Y$ denote the task outcome.
For analytical convenience, we consider the following relaxed optimization objective for the Manager Agent as
$
J(g) = \mathbb{E}[R(Y)] - \gamma \cdot \mathbb{E}[|Z|],
$
where $R(\cdot)$ denotes the task reward.

\begin{assumption}
The achievable task performance depends monotonically on the mutual information between $Z$ and $Y$ as
$
\mathbb{E}[R(Y)] = f(I(Z;Y)),  f'(\cdot) > 0.
$
\end{assumption}

As the mutual information $I(Z;Y)$ increases, the compressed context preserves more task-relevant information, enabling more accurate downstream decision-making.
Such monotonic relationships between mutual information and task performance are widely adopted in representation learning and decision-making theory~\cite{DBLP:conf/iclr/AlemiFD017}.

\begin{assumption}
The expected token cost of maintaining context $Z$ is proportional to its entropy as
$
\mathbb{E}[|Z|] = \kappa H(Z)$, where $\kappa > 0$.

\end{assumption}

In LLM-based systems, the entropy of a context correlates with its expected token usage.
Modeling context cost via entropy therefore provides a principled abstraction of computational overhead.

\begin{theorem}
\label{theorem2}
Under Assumptions 5-6, optimizing the Manager Agent’s objective is equivalent to solving the Information Bottleneck problem as
$
\min_{p(z|c)} I(Z;C) - \lambda I(Z;Y)$, where $\lambda > 0.
$
\end{theorem}

\begin{proof}
By Assumption 6 and the monotonicity of $f$, maximizing $J(g)$ is equivalent to maximizing
$
\alpha I(Z;Y) - \beta H(Z), \quad \alpha,\beta > 0.
$
Since $Z = g(C)$ is a deterministic function of $C$,
$
I(Z;C) = H(Z).
$
Thus, the objective becomes
$
\alpha I(Z;Y) - \beta I(Z;C).
$
Letting $\lambda = \alpha / \beta$ yields the Information Bottleneck objective.
\end{proof}

\eat{
Theorem~\ref{theorem2} shows that the Manager Agent implicitly performs Information Bottleneck optimization by compressing the global context to discard task-irrelevant information while preserving information that is predictive of task outcomes. This result provides an information-theoretic explanation for the design of the Manager Agent’s optimization objective, indicating that effective context compression emerges naturally from optimizing decision quality under limited representational capacity rather than from heuristic summarization alone.}

Based on the solution of the Information Bottleneck problem, we obtain the following corollary.

\begin{corollary}
The optimal compression distribution satisfies
\[
p(z|c)
\propto
p(z)\exp\left(
-\lambda D_{\mathrm{KL}}(p(y|c)\,\|\,p(y|z))
\right).
\]
\end{corollary}

\section{Benchmarks}
\label{appendix: bench}
\begin{table}[h]
\centering
\caption{Statistics of evaluation benchmarks.}
\label{tab:benchmark_stats}
\begin{tabular}{c c}
\hline
\textbf{Benchmark} & \textbf{\# Samples} \\
\hline
DABench            & 257 \\
DA-Code            & 500 \\
ScienceAgentBench  & 102 \\
MLE-Dojo           & 10  \\
\hline
\end{tabular}
\end{table}

\begin{itemize}
\item \textbf{DABench}~\cite{DBLP:conf/icml/HuZWCM0WSXZCY0K24}.  
DABench (also referred to as InfiAgent-DABench) is a benchmark specifically designed to evaluate LLM-based agents on end-to-end data analysis tasks. It consists of 257 data analysis questions derived from 52 real-world CSV files, where each task requires complex reasoning over tabular data and interaction with an executable environment. To enable automatic evaluation of open-ended outputs, DABench adopts a format-prompting strategy that standardizes results into a closed form.

\item \textbf{DA-Code}~\cite{DBLP:conf/emnlp/HuangLYZLWHHLZL24}.  
DA-Code is a code generation benchmark tailored to LLM-based agents, targeting realistic data science workflows. It comprises 500 tasks collected from diverse real-world data sources, covering multiple stages of the data science pipeline, including data wrangling, exploratory data analysis, and machine learning. All tasks are situated in an executable environment that supports interactive agent execution, and evaluation focuses on whether the generated code correctly fulfills the specified analysis objectives.

\item \textbf{ScienceAgentBench}~\cite{DBLP:conf/iclr/ChenCNZWYLLWLDX25}.  
ScienceAgentBench evaluates LLM-based agents in the context of data-driven scientific discovery. Unlike general-purpose coding or data analysis benchmarks, it extracts 102 tasks from 44 peer-reviewed publications spanning four scientific disciplines (Bioinformatics, Computational Chemistry, Geographical Information Science, and Psychology \& Cognitive Neuroscience). Each task is validated by domain experts and implemented as a self-contained Python program, covering key scientific activities such as data processing, model development, data analysis, and visualization within authentic research workflows.

\item \textbf{MLE-Dojo}~\cite{qiang2025mledojo}.  
MLE-Dojo introduces an interactive, Gym-style environment for training and benchmarking autonomous LLM agents on realistic machine learning engineering (MLE) workflows. Built upon a curated collection of over 200 real-world Kaggle challenges, MLE-Dojo includes tasks such as data preprocessing, model architecture design, hyperparameter tuning, and iterative debugging. Its executable, multi-step framework enables structured experimentation with feedback loops and supports rigorous evaluation under practical engineering settings. Due to the high computational cost of full machine learning pipelines, we randomly sample 10 tasks from MLE-Dojo for evaluation.

\end{itemize}

\section{Baselines}
\label{appendix: baselines}
\begin{itemize}
\item \textbf{AutoGen}~\cite{wu2024autogen}.
AutoGen is an open-source multi-agent framework that enables the composition of multiple LLM-based agents for collaborative problem solving. Agents within AutoGen communicate through natural language or executable code, enabling flexible interaction patterns and tool usage. The framework serves as a general infrastructure for building LLM applications of varying complexity. Following the implementation of DSBench~\cite{DBLP:conf/iclr/JingHWYYM0DY25}, we construct a data science agent based on AutoGen as one of our baselines.

\item \textbf{ReAct}~\cite{DBLP:conf/iclr/YaoZYDSN023}.
ReAct introduces a prompting paradigm that tightly couples reasoning and acting within LLMs. Unlike conventional chain-of-thought methods that focus solely on reasoning traces, ReAct interleaves intermediate reasoning steps with task-oriented actions, enabling dynamic planning, external information querying, and adaptive strategy updates. In our experiments, we design specific prompts to adapt ReAct to data science scenarios.

\item \textbf{Code Interpreter}~\cite{openai2023}.
The Code Interpreter, developed by OpenAI, enables LLMs to generate and execute Python code within a secure, sandboxed environment. This capability allows models to perform complex computations, data processing, format conversion, and visualization through iterative code execution. By incorporating real-time execution feedback, the Code Interpreter enhances structured analytical reasoning beyond pure natural language generation.

\item \textbf{LAMBDA}~\cite{Sun17072025}.
LAMBDA is an open-source multi-agent data analysis system for solving data-centric analytical tasks via natural language interaction, without requiring explicit programming from users. It adopts a dual-agent architecture, where a programmer agent translates user intent into executable code, and an inspector agent performs debugging and iterative refinement to improve robustness and correctness.

\item \textbf{Data Interpreter}~\cite{DBLP:conf/acl/HongLLLWZLCZWZZ25}.
Data Interpreter is an LLM-based autonomous data science agent that aims to solve end-to-end data analysis tasks. It employs hierarchical graph modeling to decompose complex problems into structured subtasks and leverages programmable node generation for iterative solution refinement and validation. The agent dynamically adapts to evolving data dependencies and integrates external tools to enhance code generation reliability.

\item \textbf{DeepAnalyze}~\cite{DBLP:journals/corr/abs-2510-16872}.
DeepAnalyze proposes an agentic LLM framework for autonomous data science workflows, supporting the complete analytical pipeline from raw data ingestion to report generation. It follows a curriculum-based agent training paradigm that progressively integrates domain-specific capabilities. The model synthesizes high-quality training data via a data-grounded trajectory construction framework and is further optimized through reinforcement learning.

\item \textbf{DataMind}~\cite{DBLP:journals/corr/abs-2509-25084}.
DataMind presents a scalable training framework for generalist data-analytic agents, addressing challenges such as limited data diversity, unstable multi-turn execution, and insufficient task grounding. The approach combines a fine-grained task taxonomy, knowledge-augmented trajectory sampling, and a hybrid training objective that integrates supervised learning with reinforcement learning. Agents trained under this framework achieve state-of-the-art performance across multiple data science benchmarks.

\item \textbf{LATM}~\cite{DBLP:conf/iclr/Cai00CZ24}. LATM is a framework that extends LLM agents with autonomous tool creation and reuse capabilities. Instead of relying solely on a fixed toolset, LATM enables agents to generate executable programs as reusable tools for solving complex tasks more effectively. The framework separates tool creation and tool usage into different roles, allowing agents to iteratively accumulate reusable tools and improve problem-solving efficiency.

\item \textbf{ML-Master2}~\cite{DBLP:journals/corr/abs-2601-10402}. ML-Master2 is an autonomous machine learning engineering agent designed for ultra-long-horizon tasks. It introduces a Hierarchical Cognitive Caching architecture that progressively distills execution experiences into reusable knowledge, enabling sustained exploration, long-term planning, and effective context management in complex machine learning workflows.

\end{itemize}

\section{Skills Used for EvoDS}
In this section, we present the predefined skills used in EvoDS, which are represented as executable tools and organized according to the hierarchical multi-agent architecture. Specifically, the skill suite covers the Manager Agent and each specialized sub-agent, enabling coordinated execution of data science workflows.

\subsection{Manager Agent}

The Manager Agent is equipped with general-purpose and agent-level tools that support task decomposition, execution control, and cross-agent coordination.

\begin{itemize}
\item \textit{data\_cleaning}: Routes data cleaning tasks to the Cleaner Agent.
\item \textit{feature\_engineering}: Routes feature engineering tasks to the Featurizer Agent.
\item \textit{model\_development}: Routes model development tasks to the Modeler Agent.
\item \textit{visualization}: Routes visualization tasks to the Visualizer Agent.
\item \textit{debugging}: Routes debugging tasks to the Debugger Agent.
\item \textit{bash}: Executes Bash programs.
\item \textit{sql}: Executes SQL programs.
\item \textit{python}: Executes Python programs.
\item \textit{context\_summarize}: Compresses long interaction histories into high-value summaries.
\end{itemize}

\subsection{Cleaner Agent}
The Cleaner Agent focuses on data preprocessing and quality improvement by applying standard data cleaning operations.

\begin{itemize}
\item \textit{fill\_missing\_values}: Fills missing entries using imputation strategies.
\item \textit{remove\_columns\_with\_missing\_data}: Drops features with excessive missing values.
\item \textit{detect\_and\_handle\_outliers\_zscore}: Identifies and treats outliers based on Z-score statistics.
\item \textit{detect\_and\_handle\_outliers\_iqr}: Detects and handles outliers using the interquartile range method.
\item \textit{remove\_duplicates}: Removes duplicate records from the dataset.
\item \textit{convert\_data\_types}: Converts columns to appropriate data types.
\item \textit{format\_datetime}: Standardizes datetime representations for temporal features.
\item \textit{data\_cleaning\_tool\_creation}: Synthesizes new data cleaning tools when predefined tools cannot solve the given task.
\end{itemize}

\subsection{Featurizer Agent}
The Featurizer Agent performs feature transformation, encoding, selection, and dimensionality reduction.

\begin{itemize}
\item \textit{one\_hot\_encode}: Encodes categorical features using one-hot representations.
\item \textit{label\_encode}: Applies ordinal encoding to categorical variables.
\item \textit{frequency\_encode}: Encodes categories based on their occurrence frequency.
\item \textit{target\_encode}: Encodes categorical features using target-conditioned statistics.
\item \textit{correlation\_feature\_selection}: Selects features based on correlation analysis.
\item \textit{variance\_feature\_selection}: Removes low-variance features.
\item \textit{scale\_features}: Normalizes or standardizes numerical features.
\item \textit{perform\_pca}: Reduces feature dimensionality via principal component analysis.
\item \textit{perform\_rfe}: Performs recursive feature elimination.
\item \textit{create\_polynomial\_features}: Generates polynomial feature expansions.
\item \textit{create\_feature\_combinations}: Constructs interaction features across multiple variables.
\item \textit{feature\_engineering\_tool\_creation}: Synthesizes feature engineering tools when predefined tools cannot solve the given task.
\end{itemize}

\subsection{Modeler Agent}

The Modeler Agent is responsible for training, tuning, and evaluating machine learning models for various machine learning tasks.

\begin{itemize}
\item \textit{logistic\_regression}: Trains logistic regression models for binary or multiclass classification.
\item \textit{linear\_regression}: Fits linear regression models for continuous target prediction.
\item \textit{random\_forest\_regression}: Trains random forest models for regression tasks.
\item \textit{random\_forest\_classification}: Trains random forest models for classification tasks.
\item \textit{xgboost\_regression}: Applies XGBoost for regression modeling.
\item \textit{xgboost\_classification}: Applies XGBoost for classification modeling.
\item \textit{lightgbm\_regression}: Trains LightGBM models for regression.
\item \textit{lightgbm\_classification}: Trains LightGBM models for classification.
\item \textit{catboost\_regression}: Applies CatBoost for regression tasks.
\item \textit{catboost\_classification}: Applies CatBoost for classification tasks.
\item \textit{machine\_learning\_tool\_creation}: Synthesizes machine learning tools when predefined tools cannot solve the given task.
\end{itemize}

\subsection{Visualizer Agent}

The Visualizer Agent generates visual representations to support data exploration and result analysis.

\begin{itemize}
\item \textit{plot\_line}: Generates line plots for trend analysis.
\item \textit{plot\_bar}: Creates bar charts for categorical comparisons.
\item \textit{plot\_histogram}: Visualizes value distributions using histograms.
\item \textit{plot\_boxplot}: Produces boxplots for statistical summary and outlier inspection.
\item \textit{plot\_scatter}: Generates scatter plots for relationship analysis.
\item \textit{plot\_heatmap}: Visualizes correlation matrices or intensity-based data.
\item \textit{plot\_pie}: Creates pie charts for proportional analysis.
\item \textit{plot\_pairplot}: Generates pairwise feature plots for exploratory analysis.
\item \textit{visualization\_tool\_creation}: Synthesizes visualization tools when predefined tools cannot solve the given task.
\end{itemize}

\section{Prompts Used for EvoDS}
\label{sec:appendix_prompts}
In this section, we present the prompts used in EvoDS to coordinate agent behaviors and facilitate effective task execution. Specifically, we describe the system prompts for the Manager Agent and each specialized sub-agent, which define their roles and responsibilities within the hierarchical multi-agent architecture. We further present the input prompt templates used for each specialized sub-agent. In addition, we detail the prompt design for extracting and formalizing the configurations of synthesized tools, enabling reusable and structured tool invocation during execution.

\begin{tcolorbox}[colframe=blue!20!gray, colback=blue!5!white, coltitle=white, fonttitle=\bfseries, title=Manager Agent System Prompt, breakable]
You are a data science expert. You excel at solving data-related problems. You are working in a Bash environment with all necessary Python libraries installed. You are starting in a directory, which contains all the data needed for your tasks. You need to utilize available tools provided to solve the given task. The maximum number of steps you can take is **\{max\_steps\}**.\\
\\
\# GLOBAL INFORMATION \#\\
The global context contains the overall task objective. Always use this global information to guide planning and execution.\\
\\
\# NOTICE \#\\
1. You should first understand the environment and conduct data analysis on the given data before handling the task.\\
2. You can't take some problems for granted. For example, you should check the existence of files before reading them.\\
3. You are restricted to operating solely within the current directory. Any attempt to save files or code outside of this directory is prohibited.\\
4. The LLM-based agent tools 'data\_cleaning', 'feature\_engineering', 'model\_development', 'visualization', and 'debugging' can access the global information and overall task objective. When calling these tools, you must clearly specify the concrete local objective, subtask requirements, and relevant context for the current step.\\
5. If the tool execution fails, you should analyze the error and try to solve it.\\
6. For challenging tasks like ML, you may need to verify the correctness of the method by checking the accuracy or other metrics, and try to optimize the method.\\
7. Before finishing the task, ensure all instructions are met and verify the existence and correctness of any generated files.\\
8. After completing the task, **directly provide the answer or the location where the result is saved, following the task requirements precisely.** Ensure the output strictly adheres to the required format without any additional explanations. Refrain from using any further tools.\\
9. If the interaction history becomes excessively long or contains redundant information, **use the `context\_summarize` tool to compress and retain only the essential context** required for completing the task, ensuring that critical constraints, decisions, and intermediate results are preserved.\\
\\
**Important:** Do not output the entire dataset content to avoid excessive context length.\\
**Important:** If multiple steps fail, **try alternative strategies** to overcome the issue, rather than repeating the same steps.

\end{tcolorbox}
\noindent\begin{minipage}{\linewidth}
\captionof{figure}{The system prompt used for the Manager agent.}
\end{minipage}

\begin{tcolorbox}[colframe=blue!20!gray, colback=blue!5!white, coltitle=white, fonttitle=\bfseries, title=Cleaner Agent System Prompt, breakable]
You are a data science expert specializing in data cleaning tasks. You are working as a sub-agent in a multi-agent system.\\
\\
\# GLOBAL CONTEXT \#\\
The global context contains the overall task objective. You should use this information to understand the broader goal and maintain consistency with the overall workflow.\\
\\
However, your responsibility is only to solve the assigned data cleaning subtask rather than the entire pipeline.\\
\\
You have access to a set of tools that can help solve data cleaning tasks. When given a dataset path and a data cleaning task description, your first step is to determine whether the provided tools can directly solve the assigned subtask. If they can, use the appropriate tool to perform the cleaning. The tool will automatically save the cleaned dataset.\\
\\
If the existing tools cannot directly solve the subtask, use the `data\_cleaning\_tool\_creation` tool to create a new tool based on the provided task description.\\
\\
The created tool should strictly follow the format below:\\
\text{\textasciigrave\textasciigrave\textasciigrave}python\\
def tool\_name(parameters):\\
\text{\ \ \ \ }\# detail of the code\\
\\
\# execute the tool\\
if \_\_name\_\_ == "\_\_main\_\_":\\
\text{\ \ \ \ }parameters = \{\ldots\}\\
\text{\ \ \ \ }tool\_name(parameters)\\
\text{\textasciigrave\textasciigrave\textasciigrave}

\end{tcolorbox}
\noindent\begin{minipage}{\linewidth}
\captionof{figure}{The system prompt used for the Cleaner agent.}
\end{minipage}

\begin{tcolorbox}[colframe=blue!20!gray, colback=blue!5!white, coltitle=white, fonttitle=\bfseries, title=Featurizer Agent System Prompt, breakable]
You are a data science expert specializing in feature engineering tasks. You are working as a sub-agent in a multi-agent system.\\
\\
\# GLOBAL CONTEXT \#\\
The global context contains the overall task objective. You should use this information to understand the broader goal and maintain consistency with the overall workflow.\\
\\
However, your responsibility is only to solve the assigned feature engineering subtask rather than the entire pipeline.\\
\\
You have access to a set of tools that can help solve feature engineering tasks. When given a dataset path and a feature engineering task description, your first step is to determine whether the provided tools can directly solve the assigned subtask. If they can, use the appropriate tool to perform the feature engineering. The tool will automatically save the processed dataset.\\
\\
If the existing tools cannot directly solve the subtask, use the `feature\_engineering\_tool\_creation` tool to create a new tool based on the provided task description.\\
\\
The created tool should strictly follow the format below:\\
\text{\textasciigrave\textasciigrave\textasciigrave}python\\
def tool\_name(parameters):\\
\text{\ \ \ \ }\# detail of the code\\
\\
\# execute the tool\\
if \_\_name\_\_ == "\_\_main\_\_":\\
\text{\ \ \ \ }parameters = \{\ldots\}\\
\text{\ \ \ \ }tool\_name(parameters)\\
\text{\textasciigrave\textasciigrave\textasciigrave}

\end{tcolorbox}
\noindent\begin{minipage}{\linewidth}
\captionof{figure}{The system prompt used for the Featurizer agent.}
\end{minipage}

\begin{tcolorbox}[colframe=blue!20!gray, colback=blue!5!white, coltitle=white, fonttitle=\bfseries, title=Modeler Agent System Prompt, breakable]
You are a data science expert specializing in machine learning and deep learning tasks. You are working as a sub-agent in a multi-agent system.\\
\\
\# GLOBAL CONTEXT \#\\
The global context contains the overall task objective. You should use this information to understand the broader goal and maintain consistency with the overall workflow.\\
\\
However, your responsibility is only to solve the assigned modeling subtask rather than the entire pipeline.\\
\\
You have access to a set of tools that can help solve machine learning and deep learning tasks. When given dataset paths and a modeling task description, your first step is to determine whether the provided tools can directly solve the assigned subtask. If they can, use the appropriate tool to complete the task. The tool will automatically save the submission file.\\
\\
If the existing tools cannot directly solve the subtask, use the `machine\_learning\_tool\_creation` tool to create a new tool based on the provided task description.\\
\\
The created tool should strictly follow the format below:\\
\text{\textasciigrave\textasciigrave\textasciigrave}python\\
def tool\_name(parameters):\\
\text{\ \ \ \ }\# detail of the code\\
\\
\# execute the tool\\
if \_\_name\_\_ == "\_\_main\_\_":\\
\text{\ \ \ \ }parameters = \{\ldots\}\\
\text{\ \ \ \ }tool\_name(parameters)\\
\text{\textasciigrave\textasciigrave\textasciigrave}

\end{tcolorbox}
\noindent\begin{minipage}{\linewidth}
\captionof{figure}{The system prompt used for the Modeler agent.}
\end{minipage}

\begin{tcolorbox}[colframe=blue!20!gray, colback=blue!5!white, coltitle=white, fonttitle=\bfseries, title=Visualizer Agent System Prompt, breakable]
You are a data science expert specializing in data visualization tasks. You are working as a sub-agent in a multi-agent system.\\
\\
\# GLOBAL CONTEXT \#\\
The global context contains the overall task objective. You should use this information to understand the broader goal and maintain consistency with the overall workflow.\\
\\
However, your responsibility is only to solve the assigned visualization subtask rather than the entire pipeline.\\
\\
You have access to a set of tools that can help solve data visualization tasks. When given a dataset path and a visualization task description, your first step is to determine whether the provided tools can directly solve the assigned subtask. If they can, use the appropriate tool to perform the visualization. The tool will automatically save the visualization plot.\\
\\
If the existing tools cannot directly solve the subtask, use the `visualization\_tool\_creation` tool to create a new tool based on the provided task description.\\
\\
The created tool should strictly follow the format below:\\
\text{\textasciigrave\textasciigrave\textasciigrave}python\\
def tool\_name(parameters):\\
\text{\ \ \ \ }\# detail of the code\\
\\
\# execute the tool\\
if \_\_name\_\_ == "\_\_main\_\_":\\
\text{\ \ \ \ }parameters = \{\ldots\}\\
\text{\ \ \ \ }tool\_name(parameters)\\
\text{\textasciigrave\textasciigrave\textasciigrave}

\end{tcolorbox}
\noindent\begin{minipage}{\linewidth}
\captionof{figure}{The system prompt used for the Visualizer agent.}
\end{minipage}

\begin{tcolorbox}[colframe=blue!20!gray, colback=blue!5!white, coltitle=white, fonttitle=\bfseries, title=Debugger Agent System Prompt, breakable]
You are a data science expert specializing in debugging code and troubleshooting issues in data science tasks. Your responsibility is to identify errors, inefficiencies, or bugs in the given code or process and provide solutions to fix them. You should analyze the task, locate potential issues, and suggest or implement fixes while ensuring the solution follows best practices.

\end{tcolorbox}
\noindent\begin{minipage}{\linewidth}
\captionof{figure}{The system prompt used for the Debugger agent.}
\end{minipage}

\begin{tcolorbox}[colframe=blue!20!gray, colback=blue!5!white, coltitle=white, fonttitle=\bfseries, title=Cleaner Agent Input Prompt, breakable]
You are given a dataset located at \{dataset\_file\}. Your task is to clean the dataset according to the following requirements:\\
\\
\{global\_task\}\\
\\
\# DATA CLEANING TASK \#\\
\{task\}\\
\\
After cleaning, save the cleaned dataset to \{saved\_dataset\_file\}.

\end{tcolorbox}
\noindent\begin{minipage}{\linewidth}
\captionof{figure}{The input prompt used for the Cleaner agent.}
\end{minipage}

\begin{tcolorbox}[colframe=blue!20!gray, colback=blue!5!white, coltitle=white, fonttitle=\bfseries, title=Featurizer Agent Input Prompt, breakable]
You are given a dataset located at \{dataset\_file\}. Your task is to process the dataset according to the following requirements:\\
\\
\{global\_task\}\\
\\
\# FEATURE ENGINEERING TASK \#\\
\{task\}\\
\\
After processing, save the processed dataset to \{saved\_dataset\_file\}.

\end{tcolorbox}
\noindent\begin{minipage}{\linewidth}
\captionof{figure}{The input prompt used for the Featurizer agent.}
\end{minipage}

\begin{tcolorbox}[colframe=blue!20!gray, colback=blue!5!white, coltitle=white, fonttitle=\bfseries, title=Modeler Agent Input Prompt, breakable]
You are given a training dataset located at \{train\_dataset\_path\}, and a testing dataset located at \{test\_dataset\_path\}. Your task is to solve the machine learning task below:\\
\\
\{global\_task\}\\
\\
\# MODELING TASK \#\\
\{task\}\\
\\
If you create a new tool to solve the task, you should strictly follow the instruction below:\\
\#\#\# Instructions:\\
1. **Load the Dataset**: Start by loading the dataset.\\
2. **Design the Model**: Based on the specified task, choose an appropriate machine learning or deep learning model.\\
3. **Train the Model**: Train the chosen model using the provided data. Ensure that the model is optimized and tuned for better performance.\\
4. **Validate the Model**: Evaluate the model's performance using suitable metrics (e.g., accuracy, F1 score, RMSE, etc.) on a validation set.\\
5. **Print Results**: Print the model's performance metrics (e.g., accuracy, loss, etc.) for inspection. This will allow further iteration on model design or training if necessary.\\
6. **Make Predictions**: Use the trained model to make predictions on the specified test data.\\
7. **Save the Results**: After prediction, save the prediction results as specified.\\
\\
Please ensure that:\\
- The model is appropriately designed and trained according to the task.\\
- The predictions are saved in the correct format.\\
- The printed results are clear and can be used for further iteration of the model.

\end{tcolorbox}
\noindent\begin{minipage}{\linewidth}
\captionof{figure}{The input prompt used for the Modeler agent.}
\end{minipage}

\begin{tcolorbox}[colframe=blue!20!gray, colback=blue!5!white, coltitle=white, fonttitle=\bfseries, title=Visualizer Agent Input Prompt, breakable]
You are given a dataset located at \{dataset\_file\}. Your task is to process the dataset according to the following requirements:\\
\\
\{global\_task\}\\
\\
\# VISUALIZATION TASK \#\\
\{task\}\\
\\
After visualization, save the visualization plot to \{saved\_plot\_file\}.

\end{tcolorbox}
\noindent\begin{minipage}{\linewidth}
\captionof{figure}{The input prompt used for the Visualizer agent.}
\end{minipage}

\begin{tcolorbox}[colframe=blue!20!gray, colback=blue!5!white, coltitle=white, fonttitle=\bfseries, title=Debugger Agent Input Prompt, breakable]
However, there are some bugs in the code. Here is the execution result:\\
\# Execution Result:\\
\{observation\}\\
\\
---\\
\\
Based on the provided execution result, please revise the script to fix these bugs. Your task is to address the error indicated in the result, and refine or modify the code as needed to ensure it works correctly.\\
\\
The Python code should strictly follow the format below:\\
\text{\textasciigrave\textasciigrave\textasciigrave}python\\
\# Provide the corrected python code here.\\
\text{\textasciigrave\textasciigrave\textasciigrave}

\end{tcolorbox}
\noindent\begin{minipage}{\linewidth}
\captionof{figure}{The input prompt used for the Debugger agent.}
\end{minipage}

\begin{tcolorbox}[colframe=blue!20!gray, colback=blue!5!white, coltitle=white, fonttitle=\bfseries, title=Tool Configurations Extraction Prompt, breakable]
You are a **Tool Configuration Extraction Assistant**.\\
\\
You are provided with the source code of a single function as:
\{code\}\\
\\
Your task is to analyze the function definition and implementation, and extract a corresponding tool configuration strictly following the format below:\\
\\
\text{\textasciigrave\textasciigrave\textasciigrave}json
\{%
\{\\
\text{\ \ \ }"type": "function",\\
\text{\ \ \ }"function": \{\{\\
\text{\ \ \ \ \ \ }"name": "tool name",\\
\text{\ \ \ \ \ \ }"description": "tool description.",\\
\text{\ \ \ \ \ \ }"parameters": \{\{\\
\text{\ \ \ \ \ \ \ \ \ }"type": "object",\\
\text{\ \ \ \ \ \ \ \ \ }"required": ["required parameters"],\\
\text{\ \ \ \ \ \ \ \ \ }"properties": \{\{\\
\text{\ \ \ \ \ \ \ \ \ \ \ \ }"parameter name": \{\{\\
\text{\ \ \ \ \ \ \ \ \ \ \ \ \ \ \ }"type": "parameter type",\\
\text{\ \ \ \ \ \ \ \ \ \ \ \ \ \ \ }"description": "parameter description"\\
\text{\ \ \ \ \ \ \ \ \ \ \ \ \}\}}\\
\text{\ \ \ \ \ \ \ \ \ \}\}}\\
\text{\ \ \ \ \ \ \}\}}\\
\text{\ \ \ \ \}\}}\\
\}\}\\
\text{\textasciigrave\textasciigrave\textasciigrave}\\
\\
**Extraction Rules**\\
The tool name must exactly match the function name.\\
The tool description must concisely describe the purpose and behavior of the function.\\
Parameters must be inferred from the function signature.\\
Required parameters are those without default values.\\
Optional parameters are those with default values.\\
\\
**Output Requirements**\\
Output **only** the tool configuration in valid JSON as:\\
\text{\textasciigrave\textasciigrave\textasciigrave}json\\
content.\\
\text{\textasciigrave\textasciigrave\textasciigrave}\\
Do NOT include explanations, comments, or markdown.\\
Do NOT include any content outside the JSON object.\\
Follow the exact schema and field names shown above.\\
\\
**Additional Constraints**\\
Do NOT invent parameters that are not present in the function.\\
Do NOT omit parameters defined in the function signature.\\
If the function takes no parameters, use:\\
\\
\text{\textasciigrave\textasciigrave\textasciigrave}\\
"required": [],\\
"properties": \{\{\}\}\\
\text{\textasciigrave\textasciigrave\textasciigrave}\\
\\
Generate the tool configuration now.

\end{tcolorbox}
\noindent\begin{minipage}{\linewidth}
\captionof{figure}{The prompt used for extracting the configurations of synthesized tools.}
\end{minipage}

\onecolumn
\end{document}